\documentclass{article}
\usepackage{graphicx} % Required for inserting images
\usepackage{amsthm}
\usepackage[american]{babel}
\usepackage{natbib} % has a nice set of citation styles and commands
\usepackage{mathtools} % amsmath with fixes and additions
\usepackage{booktabs} % commands to create good-looking tables
\usepackage{tikz} % nice language for creating drawings and diagrams
\usepackage{makecell}
\usepackage{authblk}
\usepackage[top=2cm, bottom=2cm, left=2cm, right=2cm]{geometry} % Define as margens do documento
\usepackage{enumitem}
\usepackage{url}
\usepackage{hyperref}
\usepackage[toc,page]{appendix}

\title{Locus - Arxiv}

\author[1, 2]{Matheus Barreto}
\author[2]{Mário de Castro}
\author[1]{Thiago R. Ramos}
\author[3]{Denis Valle}
\author[1]{Rafael Izbicki}

\affil[1]{\small Department of Statistics, Federal University of São Carlos, São Carlos, São Paulo, Brazil}
\affil[2]{\small Institute of Mathematics and Computer Science, University of São Paulo, São Carlos, São Paulo, Brazil}
\affil[3]{\small School of Forest, Fisheries, and Geomatics Sciences, University of Florida, Gainesville, Florida, United States of America}

\usepackage[most]{tcolorbox}
\tcbset{
  colback=gray!2,
  colframe=gray!35,
  arc=2mm,
  boxrule=0.4pt,
  left=6pt,right=6pt,top=6pt,bottom=6pt
}

\newtheorem{theorem}{Theorem}
\newtheorem{remark}{Remark}

\newtheorem{assumption}{Assumption}
\usepackage{amssymb}

\usepackage{amsfonts} % provides \mathbb
\DeclareMathOperator{\Prbb}{\mathbb{P}}
\DeclareMathOperator{\E}{\mathbb{E}}
\newcommand{\1}{\mathbf{1}}

\newcommand{\methodstyle}[1]{{\textsc{#1}}}

\newcommand{\locus}{\methodstyle{Locus}}

\newcommand{\locusT}{\methodstyle{Locus-$\tau$}}

\newcommand{\locusAlpha}{\methodstyle{Locus-$\alpha$}}

\newcommand{\locusLambda}{\methodstyle{Locus-Tuned}}

\title{\locus{}: A Distribution-Free Loss-Quantile Score for  Risk-Aware Predictions}

\begin{document}
\maketitle

\begin{abstract}
Modern machine learning models can be accurate on average yet still make
 mistakes that dominate deployment cost. We introduce \locus{}, a
distribution-free wrapper that produces a per-input loss-scale reliability
score for a fixed prediction function. Rather than quantifying uncertainty about
the label, \locus{} models the realized loss of the prediction function  using any  engine that outputs a
predictive distribution for the loss given an input. A simple split-calibration step
turns this function into a distribution-free interpretable score that is comparable across
inputs and can be read as an upper loss level. The score is useful
on its own for ranking, and it can optionally be
thresholded to obtain a transparent flagging rule with distribution-free control of large-loss events.
Experiments across 13 regression benchmarks show that \locus{} yields effective
risk ranking and reduces large-loss frequency compared to standard
heuristics.
\end{abstract}

\section{Introduction}\label{sec:intro}

Modern machine learning models can be highly accurate on average \citep{lecun2015deep,lakshminarayanan2017simple}, yet a deployed
system is often judged one decision at a time. For a given input $x$, the operational
question is not only what the model predicts, but whether the prediction is
safe to act on or should be flagged for additional review. This need
arises across domains such as clinical decision support, credit
scoring, and autonomous systems, where rare but large errors dominate the real cost of
automation.

A central difficulty is that prediction errors are inherently heterogeneous.
Even with abundant data, irreducible (aleatoric) randomness implies that
the outcome $Y$ is not a deterministic function of $X$
\citep{kendall2017uncertainties}. With limited data, (epistemic)
uncertainty further amplifies risk in regions where the system has not been
well observed \citep{hullermeier2021aleatoric}. Consequently, any fixed predictor $g(x)$ can incur occasional
catastrophic losses on inputs that look benign under global performance
metrics.

Standard performance summaries such as accuracy, RMSE, ROC curves, and AUC quantify the average performance of a model over the deployment distribution \citep{fawcett2006introduction}. They are useful for model selection and benchmarking, but they are global by construction: they do not tell us how risky a single prediction $g(x)$ is. Calibration tools such as reliability diagrams, expected calibration error, or conformal prediction sets are important improvements by providing metrics that are more local, yet they still aggregate over many inputs and often do not translate into a simple rule of the form “flag this $x$ as predictions for it are unreliable”.

In practice, per-instance reliability is usually approximated by some notion of uncertainty. Practitioners estimate the conditional distribution $Y\mid X=x$ using Bayesian predictive distributions, deep ensembles, Monte Carlo dropout, prediction intervals or sets, and  conformal methods (\citealt{lakshminarayanan2017simple,angelopoulos2023gentle}; see Section~\ref{sec:related}  for related work). These approaches produce quantities such as predictive variances, entropies, or interval widths. A common heuristic is then to flag points with large estimated variability or wide prediction sets, and to trust points with tight predictive distributions.
However, these uncertainty proxies are only loosely linked to the quantity that typically matters in deployment: the loss incurred by the prediction that is actually used.

Formally, we fix a deployed prediction function $g:\mathcal{X}\to\mathcal{Y}$ and a
task loss $L(\cdot,\cdot)$ (squared error, 0--1 loss, or a domain-specific cost).
For a new pair $(X,Y)$, we define the realized loss
\[
Z \;=\; L\bigl(g(X),Y\bigr).
\]
A practitioner typically has a threshold $\tau$ that encodes \emph{unacceptable
error} (e.g., a clinically meaningful deviation). The goal is to design a
flagging rule that keeps the event $\{Z>\tau\}$ rare among predictions we choose
to trust, with guarantees that hold in finite samples and without assuming that probabilistic
models are correctly specified.

\paragraph{Our approach: loss-controlled flagging.}
We propose \locus{} (LOss Control using Uncertainty Scores), a distribution-free
wrapper that turns any estimate of the conditional loss law into a deployable
reliability indicator for a fixed prediction function $g(x)$. Concretely, given an i.i.d.\ calibration
sample and any black-box engine that outputs a predictive CDF for the realized loss, $\widetilde F(\cdot\mid x)\approx \Prbb(Z\le \cdot\mid X=x),$ 
we apply a simple split-calibration step to obtain a per-input easily interpretable score $U_\alpha(x)$.
The resulting $U_\alpha(x)$ is calibrated on the loss scale and is comparable
across inputs. When a loss tolerance $\tau$ is specified, we accept (do not flag) an input
whenever $U_\alpha(x)\le \tau$ and flag otherwise, yielding a transparent rule with
 distribution-free control of large-loss events (Theorem~\ref{thm:loss-control-A}).
 
Figure~\ref{fig:intro-example} shows why predictive variance (panel~b) can be a poor proxy for deployment risk. For two fixed prediction functions $g(x)$ (linear, left; $k$NN, right), panel~(a) plots the data and fitted $g(x)$, panel~(c) shows the realized loss $Z=L(g(X),Y)$ together with our loss-scale calibrated upper bound $U_\alpha(x)$ (an estimated $(1-\alpha)$ loss quantile), and panel~(d) thresholds $U_\alpha$ to accept/flag inputs for this specific $g$. Interestingly, the linear model incurs a large loss in low-variance regions due to misfit, which $\sigma(x)$ does not capture, but $U_\alpha(x)$ does.

\begin{figure*}[h]
  \centering
\includegraphics[scale=0.28]{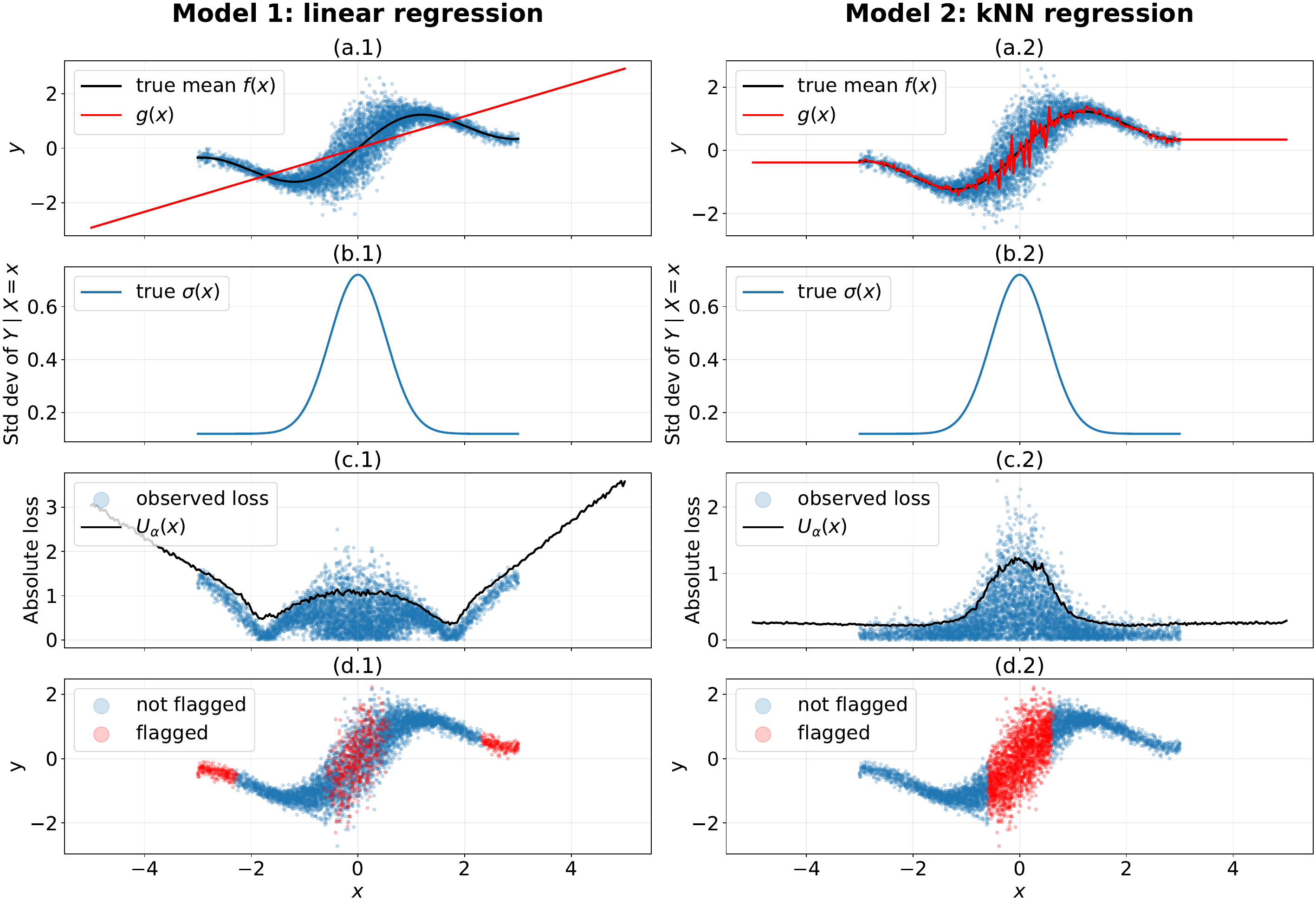}
  \caption{\textbf{Loss-centric reliability versus variance.}
We fit two models to the same data: linear regression (left) and $k$NN regression (right). (a) Data with the true mean $f(x)$ (black) and fitted prediction function $g(x)$ (red). (b) True conditional standard deviation $\sigma(x)$ (aleatoric proxy). (c) Realized absolute loss $Z=|g(X)-Y|$ (points) and our \locus{} score $U_\alpha(x)$ (curve), a calibrated estimate of the $(1-\alpha)$ loss quantile. (d) The induced accept/flag rule for this specific prediction function: accept if $U_\alpha(x)\le\tau$ (blue) and flag otherwise (red). $U_\alpha$ is interpretable in loss units and captures failures in the predicted values (e.g., linear misfit in low-variance regions) that $\sigma(x)$ can miss.}\label{fig:intro-example}
\end{figure*}

\subsection{Novelty and Contributions}\label{subsec:novelty}

\locus{} is a \emph{loss-centric} reliability wrapper for a fixed prediction function $g$,
built around the realized loss $Z=L(g(X),Y)$ rather than uncertainty about $Y\mid X=x$.
 For flagging, it assumes  a user-specified tolerance $\tau$ on the loss scale, which defines what
counts as  unacceptable error.
We make the following contributions:

\begin{itemize}
  \item \textbf{A calibrated loss-quantile score   from any predictive loss CDF.}
  Given any predictive CDF $\widetilde F(\cdot\mid x)$ for $Z\mid X=x$, we build a
  local score $U_\alpha(x)$ with finite-sample, distribution-free marginal validity:
  \[
  \Prbb\bigl(Z\le U_\alpha(X)\bigr)\ge 1-\alpha
  \qquad\text{(Theorem~\ref{thm:coverage}).}
  \]
  Under additional assumptions, $U_\alpha(x)$ also achieves asymptotic conditional
  validity (Theorem~\ref{thm:asymp-coverage}), and can be interpreted as a distribution-free calibrated estimate of the upper 
$1-\alpha$ loss quantile at $x$: if $U_\alpha(x)=K$, then there's approximately $\alpha$ probability that the loss at $x$ will be larger than $K$.

  \item \textbf{A simple, interpretable flagging rule with explicit loss control.}
  Thresholding the bound yields acceptance regions
  \[
    A_{\lambda;\alpha}=\{x:U_\alpha(x)\le \lambda\}.
  \]
  Setting $\lambda=\tau$ guarantees distribution-free control
  of ``trusted-but-bad'' events:
  \[
    \Prbb\bigl(Z>\tau,\;X\in A_{\tau;\alpha}\bigr)\le \alpha
    \qquad\text{(Theorem~\ref{thm:loss-control-A})}
  \]
  This matches the operational objective of bounding the overall frequency of
  unacceptable losses among predictions we choose to trust. Moreover, we also have that $\Prbb\bigl(Z>\tau \mid X\in A_{\tau;\alpha}\bigr)\;\lesssim\;\alpha,$ i.e., the probability of a large loss in the accepted samples is approximately smaller than $\alpha$.

  \item \textbf{Practical tuning for conditional exceedance targets.}
  While our core guarantee controls the joint event above, many deployments aim to
  target a conditional exceedance level $\eta$ among accepted points,
  $\Prbb(Z>\tau\mid X\in A_{\lambda;\alpha})\approx \eta$.
  We provide a simple heuristic validation-based tuning scheme, \locusLambda{}, which tunes $\lambda$ over a grid.
  We also provide a fully distribution-free validation-based selection of $\lambda$, guaranteeing with high probability that
  $\Prbb(Z>\tau \mid X\in A_{\lambda^\star;\alpha})\le \eta$
  (Appendix~\ref{app:distribution-free}, Theorem~\ref{thm:uniform-lambda-explicit-correct}).

  \item \textbf{Epistemic-aware inflation without changing calibration.}
  We introduce default constructions of $\widetilde F$ that become more conservative
  in data-scarce regions, improving robustness in extrapolative regimes while leaving the
  distribution-free calibration step and guarantees intact.
\end{itemize}

\subsection{Relation to Other Work}
\label{sec:related}

Our method interacts with three broad lines of work:  conformal prediction and risk control, local or conditional coverage methods, and the literature on selective prediction and reject options. Here, we highlight how \locus{} differs from these approaches.

\vspace{2mm}
\textbf{Selective prediction, reject options, and flagging.}
Selective prediction (classification with a reject option) studies rules that either
predict or abstain, dating back to Chow's rule \citep{Chow70} and many cost/coverage
variants; see surveys in \citep{Zhang2023,Hendrickx2024}. Modern work often learns a
selection function (sometimes jointly with a deep predictor) to trade off coverage
against performance on the selected subset \citep{Geifman2017,Geifman2019}. We view
\locus{} as a wrapper that induces an acceptance region
$A_{\lambda;\alpha}=\{x:U_\alpha(x)\le \lambda\}$, but we do not learn a selector to
optimize a coverage--accuracy objective. Instead, we calibrate a per-input upper bound
on the realized loss and threshold it to  control of large-loss events among accepted cases.
The score $U_\alpha$ is also useful on its own for risk ranking and triage.

As noted by \citet{Hendrickx2024}, the literature on reject options for regression is more limited than its classification counterpart.  In this context, \citet{zaoui2020regression_reject}
studies budgeted rejection and obtain variance-thresholding rules, whereas \locus{} targets
the realized loss and provides distribution-free loss exceedance control.

\vspace{2mm}
\textbf{Flagging via uncertainty and OOD scores.}
A common way to provide per-instance reliability is to use a scalar
\emph{uncertainty proxy} (e.g., predictive variance/entropy, ensemble
disagreement) as a triage signal, often obtained from approximate Bayesian
methods such as MC dropout or deep ensembles
\citep{gal2016dropout,lakshminarayanan2017simple,zaoui2020regression_reject}. However, such proxies can be
misaligned with deployment risk and may fail to increase in regions with few points \citep{ovadia2019can}. 
Closely related, \emph{out-of-distribution (OOD) detection} and \emph{open-set
recognition} aim to flag atypical inputs using confidence- or distance-based
scores (e.g., maximum softmax probability, ODIN, Mahalanobis feature distances)
\citep{hendrycks2017baseline,liang2018odin,lee2018mahalanobis}, with surveys
connecting these ideas to anomaly detection and open-set recognition
\citep{geng2021openset,yang2024}. While effective at detecting atypical inputs,
these scores are typically not calibrated to directly control the realized
task loss of a fixed deployed predictor on the accepted set. \locus{}
instead models the realized loss $Z=L(g(X),Y)$ and thresholds a calibrated
upper loss bound.

\vspace{2mm}
\textbf{Conformal prediction and risk control.}
Classical conformal prediction constructs set-valued predictions with finite-sample
marginal coverage under exchangeability \citep{Vovk2005,angelopoulos2023gentle,Izbicki2025},
typically guaranteeing $\Prbb\!\bigl(Y\in C(X)\bigr)\ge 1-\alpha$ without assumptions on the
predictive model. Risk-controlling extensions calibrate prediction sets to bound a
user-chosen loss functional aggregated over future points (e.g., expected risk control) \citep{Bates2021,Angelopoulos2024}. 
These approaches remain global in the sense that they bound an average loss over all future points. 
In contrast, \locus{}  thresholds a loss bound to control catastrophic-loss events on the accepted set via
$\Prbb(Z>\tau,\;X\in A_{\tau;\alpha})\le \alpha$.

\vspace{2mm}
\textbf{Local and conditional guarantees.}
A growing literature on conformal methods aims to move beyond purely marginal coverage and to provide guarantees that adapt to the local difficulty of the prediction task. Examples include local conformal schemes based on regression trees and forests  \cite{CabezasSoaresRamosSternIzbicki2024, cabezas2025,cabezas2025cp4sbi} and other  data-driven groups for local conformal prediction \cite{Lei2014,izbicki2020flexible,izbicki2022cd,Gibbs2023}. 
Our method for building $U_\alpha(x)$ is structurally similar in that it seeks to reflect heterogeneity across $x$, but it operates on the distribution of the loss $Z = L(g(X), Y)$ rather than directly on $Y$.

\vspace{2mm}
\textbf{Epistemic-aware conformal methods.}
Recent work has started to explicitly combine aleatoric and epistemic uncertainty within conformal prediction. \texttt{EPICSCORE} introduces a general mechanism for augmenting any conformal score with a Bayesian estimate of epistemic uncertainty, producing prediction regions that automatically widen in data-sparse areas while preserving finite-sample marginal coverage and achieving asymptotic conditional coverage \citep{CruzCabezas2025}. Similarly, \textsc{LUCCa} (Local Uncertainty Conformal Calibration) targets robotics dynamics, calibrating aleatoric uncertainty estimates while accounting for epistemic uncertainty induced by changing environments \cite{Marques2024}. 
\locus{} is complementary: we can use such
epistemic-aware calibration to build a sharper, more conservative predictive CDF
for $Z\mid X=x$, but the main contribution here is creating an interpretable score to measure reliability of a prediction, plus a  flagging layer that
turns any predictive CDF for the realized loss into an acceptance rule with
explicit guarantees. 

\section{Methodology}
\label{sec:method}

Recall the deployed predictor $g$, loss $L$, and realized loss
$Z=L(g(X),Y)$ (Section~\ref{sec:intro}). Let $(X,Y)$ denote an independent
deployment draw, and let
$D=\{(X_i,Y_i)\}_{i=1}^n$ be an i.i.d.\ calibration sample from the same
deployment distribution, independent of the data used to train $g$.
Define the observed calibration losses $Z_i := L(g(X_i),Y_i)$.
Using $D$, we construct (i) a calibrated local upper bound $U_\alpha(x)$ for
$Z\mid X=x$, and (ii) flagging rules based on thresholding $U_\alpha(x)$.

Our framework has two layers. First, we build a base predictive
distribution for the loss $Z\mid X=x$ using a probabilistic model fitted on
part of the calibration data. Second, we apply a  distribution-free
calibration step on held-out data to convert the base loss predictive distribution
into a valid local upper bound $U_\alpha(x)$ that satisfies a finite-sample
marginal guarantee, regardless of whether the base model is correctly specified. In the next section we detail these steps.

\subsection{Constructing a calibrated local loss bound}
\label{sec:method-U}

\subsubsection*{Step 1: Data split}
\label{sec:method-split}

We split the calibration sample into two independent parts:
\[
D_1 = \{(X_i,Y_i)\}_{i\in I_1}
\quad \text{and}\quad
D_2 = \{(X_i,Y_i)\}_{i\in I_2},
\]
The set $D_1$ is
used to fit a predictive model for $Z\mid X$, while $D_2$ is used to calibrate.
We form the induced loss data
\[
\widetilde D_1 := \{(X_i,Z_i): i\in I_1\}
\ \text{and} \ 
\widetilde D_2 := \{(X_i,Z_i): i\in I_2\}.
\]

\subsubsection*{Step 2: Base predictive model for $Z\mid X$}
\label{sec:method-base-cdf}

Here we propose two constructions for such a model, though step 3 is agnostic to how step 2 is performed.
On $\widetilde D_1$ we fit a Bayesian model for $Z$ given $X$.
We assume a family of conditional densities
\[
\mathcal{F} = \{f(z\mid x,\theta): \theta\in\Theta\},
\]
 modeling the aleatoric variability of $Z\mid X=x$, and place a prior on $\theta$
($\Theta$ can be infinite-dimensional; i.e., this can be a nonparametric model).
After observing $\widetilde D_1$, we obtain the posterior
$f(\theta\mid \widetilde D_1)$ and the corresponding posterior predictive CDF
\begin{align}
  F_Z(z\mid x,D_1)
  &:= \Prbb(Z\le z\mid x,D_1)\notag \\
   &= \int_\Theta F_Z(z\mid x,\theta)\,f(\theta\mid \widetilde D_1)\,d\theta,
  \label{eq:predictive-cdf}
\end{align}
where $F_Z(z\mid x,\theta)$ is the CDF induced by $f(z\mid x,\theta)$.
Finally, define $
\widetilde F(\cdot\mid x)\equiv F_Z(\cdot\mid x,D_1).$

In practice, the posterior may not be computed in closed form. Instead, we can approximate
\eqref{eq:predictive-cdf} using either Monte Carlo aggregation over posterior draws
$\theta^{(1)},\dots,\theta^{(S)}\sim f(\theta\mid \widetilde D_1)$ (e.g.,\ Bayesian additive regression trees (BART),
Monte Carlo dropout, Bayesian neural networks) or  a variational approximation
(e.g.,\ variational Gaussian processes).

\paragraph{Alternative definition of $\widetilde F$.}
Posterior averaging in \eqref{eq:predictive-cdf} may fail to reflect a higher probability of loss in data-scarce
or extrapolative regions \citep{ovadia2019can,Wang2022LocalGPBART}. To explicitly encode epistemic uncertainty, we also define
a \emph{density-modulated posterior envelope} for the predictive CDF.
Given posterior draws $\{\theta^{(s)}\}_{s=1}^S$ and a trimming level
$\gamma(x)\in(0,1)$ (we discuss how to choose it later), define
\begin{equation}
\widetilde F(z\mid x)
:=
\mathrm{quantile}_{\,\gamma(x)}
\Bigl(
\bigl\{ F(z\mid x,\theta^{(s)}) \bigr\}_{s=1}^S
\Bigr).
\label{eq:inflated-cdf}
\end{equation}
For each fixed $x$, $\widetilde F(\cdot\mid x)$ is a valid  
CDF in $z$. Smaller $\gamma(x)$ corresponds to a more conservative (lower) envelope
of CDFs (heavier right tail), and therefore larger upper bounds.

\subsubsection*{Step 3: Distribution-free calibration via PIT values}
\label{sec:method-pit}

We calibrate the 
predictive CDF $\widetilde F$ on
$\widetilde D_2$ using PIT (probability integral transform) values \citep{zhao2021diagnostics,dey2025towards}.
For each $(X_i,Z_i)\in\widetilde D_2$, define $W_i := \widetilde F(Z_i\mid X_i).$ 
For a target tail level $\alpha\in(0,1)$, define the calibrated level
\begin{equation*}
  t_{1-\alpha}
  := \mathrm{quantile}_{1-\alpha}\bigl(\{W_i: (X_i,Z_i)\in\widetilde D_2\}\bigr).
  \label{eq:t-cal}
\end{equation*}
Finally, define the local performance score as the calibrated upper loss bound
\begin{equation*}
  U_\alpha(x)
  := \widetilde F^{-1}\bigl(t_{1-\alpha}\mid x\bigr).
  \label{eq:U-def}
\end{equation*}

\paragraph{Interpretation of $U_\alpha(x)$: conformalized loss quantile.}
For a fixed $\alpha$, the score $U_\alpha(x)$ is a local measure of risk:
large values signal inputs where the model may incur large loss. It is a calibrated estimate of the $1-\alpha$ conditional quantile of the loss at $x$.

The calibration step yields a finite-sample marginal guarantee, regardless of whether the base predictive model is correct. Indeed, the next theorem  shows that, on a new test point, the realized loss falls below the calibrated bound with probability at least $1-\alpha$, and not far from it.

\begin{theorem}[Marginal validity]
\label{thm:coverage}
Under the construction above, if the data points are i.i.d.,
\[
  \Prbb\bigl(Z \le U_\alpha(X)\bigr) \;\ge\; 1-\alpha.
\]
Moreover, 
if ties  between the PIT scores occur with
probability $0$, then
$$
1-\alpha+\frac{1}{|\tilde D_2|+1}\;>\;
 \Prbb\bigl(Z \le U_\alpha(X)\bigr).
$$
\end{theorem}

Moreover, under additional assumptions that the estimated loss distribution $\widetilde F(\cdot\mid x)$ becomes accurate as the calibration sample grows and that the true conditional loss behavior is well-behaved (see Appendix~\ref{app:proofs}), $U_\alpha$ is approximately the $1-\alpha$ quantile of the loss function at $x$:

\begin{theorem}[Asymptotic conditional coverage]\label{thm:asymp-coverage}
Under Assumption~\ref{ass:acc}, we have
\[
\Prbb\!\bigl(Z\le U_{\alpha}(x)\mid X=x,\; D\bigr)
\xrightarrow{\Prbb} 1-\alpha
\quad \text{as} \ n \to\infty,
\]
where the left-hand side is a random variable (randomness from $D$) and
$\xrightarrow{\Prbb}$ denotes convergence in probability with respect to $D$. 
\end{theorem}

\subsubsection{Choosing $\gamma(x)$ via kNN radius}
\label{sec:gamma-knn}

To better encode epistemic uncertainty, we let the trimming level $\gamma(x)$ depend on how well-supported $x$ is by the design points in $D_1$. Intuitively, when $x$ lies in a data-scarce region, posterior draws for $Z\mid X=x$ are less reliable, so we choose a smaller $\gamma(x)$ and take a more conservative lower envelope in \eqref{eq:inflated-cdf}, which inflates the resulting bound $U_\alpha(x)$. Below we give a simple, scalable way to compute such a $\gamma(x)$ from a local density proxy based on $k$-nearest-neighbor distances in a representation space.

Let $\phi:\mathcal X\to\mathbb R^d$ be a representation in which distances are
meaningful. For tabular data, $\phi(x)$ can be standardized covariates. For
high-dimensional inputs, we can use  an embedding from the predictor $g$
(e.g.\ penultimate layer of a neural-network) or from the model used to predict $Z\mid X$.

First, we build a  nearest neighbor index on
$\{\phi(X_i): i\in I_1\}$. Fix $k$ (e.g.\ $k=50$). For any $x$, define the kNN
radius
\[
r_k(x) := \|\phi(x)-\phi(X_{(k)}(x))\|,
\]
where $X_{(k)}(x)$ is the $k$-th nearest neighbor of $x$ among $\{X_i:i\in I_1\}$.
Larger $r_k(x)$ indicates lower local design density.
 Next, compute reference quantiles on $D_1$ via
$
q_{\mathrm{lo}} := \mathrm{quantile}_{0.50}(\{r_k(X_i): i\in I_1\})$ and
$q_{\mathrm{hi}} := \mathrm{quantile}_{0.90}(\{r_k(X_i): i\in I_1\}).$
Now, define the standardized scarcity score
\[
s(x) := \frac{r_k(x)-q_{\mathrm{lo}}}{q_{\mathrm{hi}}-q_{\mathrm{lo}}+\varepsilon},
\]
for a small value of $\varepsilon$ (we set it to $0.000001$). Then set
\begin{equation}
\gamma(x)
=
\gamma_{\max}
-
(\gamma_{\max}-\gamma_{\min})
\cdot
\sigma\!\left(\frac{s(x)-m}{s_\gamma}\right),\label{eq:gamma-logistic}
\end{equation}
where $\sigma$ is the logistic function.
 We choose $k=50$, $\gamma_{\min}=0.15$,
$\gamma_{\max}=0.9$, $m=0$, and $s_\gamma=1$. With this choice,
$\gamma(x)$ is smaller in low-density regions (large $r_k(x)$), yielding a more
conservative envelope in \eqref{eq:inflated-cdf} and larger bounds $U_\alpha(x)$.

\subsection{Flagging rules based on $U_\alpha(x)$}
\label{sec:method-flagging}

We now show how to turn the local score $U_\alpha(x)$ into practical flagging
rules that control large losses at the cohort level.
Given a threshold $\lambda\in\mathbb{R}$, we define the unflagged (accepted) region
\[
A_{\lambda;\alpha} := \{x\in\mathcal{X}: U_\alpha(x)\le \lambda\},
\]
and its complement, the flagged region
$
A_{\lambda;\alpha}^{\mathrm{c}} := \{x : U_\alpha(x) > \lambda\}.
$
%and the corresponding acceptance rate
%$
%p_{A_{\lambda;\alpha}} := \Prbb\bigl(X\in A_{\lambda;\alpha}\bigr).
%$ 
We assume the user specifies  a tolerance $\tau$ on the loss scale, which defines what
counts as unacceptable error.
 In what follows, we consider two ways of choosing $\lambda$, corresponding to two
variants of loss-controlled flagging.

\subsubsection{Flagging with $\lambda = \tau$}
\label{sec:lambda-equals-T}

The simplest choice is to set the acceptance threshold equal to the loss
threshold, $\lambda = \tau$, in which case the acceptance region is $
  A_{\tau;\alpha} := \{x: U_\alpha(x)\le \tau\}.
$ In this case we obtain the following guarantee:

\begin{theorem}
\label{thm:loss-control-A}
Under the construction above,
\[
  \Prbb\bigl(Z>\tau,\;X\in A_{\tau;\alpha}\bigr) \le \alpha.
\]
\end{theorem}

The joint bound controls the overall mass of ``trusted but bad'' predictions:
among all future cases, the fraction that are both unflagged and incur loss
above $\tau$ is at most $\alpha$.  

Moreover, the next theorem shows that, for large  calibration sample sizes,   the simple rule ``trust  the
prediction at $x$ if $U_\alpha(x)\le \tau$'' is approximately equivalent to the
oracle rule that would trust $x$ only when the probability of an unacceptable
error at $x$ is at most $\alpha$.  That is, thresholding $U_\alpha$
eventually makes the right flag decision pointwise.

\begin{theorem}
\label{thm:pointwise-decision-consistency}
Fix $\alpha\in(0,1)$, a threshold $\tau\in\mathbb{R}$, and a point $x\in\mathcal X$.
Assume that $F_{Z\mid X}(\cdot\mid x)$ is continuous and that its $(1-\alpha)$-quantile
$
q_{1-\alpha}(x):=F^{-1}_{Z\mid X}(1-\alpha\mid x)
$
is unique. Assume also that $U_{\alpha}(x)\xrightarrow{\Prbb} q_{1-\alpha}(x)$ as $n\to\infty$ (which holds under the conditions of Theorem \ref{thm:asymp-coverage})
and that  $F_{Z\mid X}(\cdot\mid x)$ is continuous.
  Then, the decision rule ``accept iff $U_{\alpha}(x)\le\tau$'' is equivalent to the oracle rule
``accept iff $\Prbb(Z>\tau\mid X=x)\le\alpha$''.
\end{theorem}

Finally, Theorem~\ref{thm:asymp-coverage} suggests that $U_\alpha(x)$ behaves like a $(1-\alpha)$ conditional loss quantile, and the selection argument in Theorem~\ref{thm:selection_tail_control} (Appendix~\ref{app:proofs}) indicates that conditioning on acceptance should preserve this tail level. Hence, we expect the exceedance probability among accepted points to be aproximately bounded by $\alpha$:
\[
\Prbb\bigl(Z>\tau \mid X\in A_{\tau;\alpha}\bigr)\;\lesssim\;\alpha.
\]
In words, the accepted set should have an approximately $\alpha$-controlled rate of large losses.

\subsubsection{\locusLambda: Tuning $\lambda$}
\label{sec:locus-lambda}

One would often also like the accepted set to satisfy a target
\emph{conditional} exceedance level $\eta$, namely
\[
q(\lambda)
:= \Prbb(Z>\tau \mid X\in A_{\lambda;\alpha})
\approx \eta.
\]
If $U_\alpha(x)$ were exactly the $(1-\alpha)$ conditional quantile of $Z\mid X=x$,
then choosing $\alpha=\eta$ and setting $\lambda=\tau$ would approximately achieve
$q(\tau)\approx \eta$. In finite samples,
however, marginal calibration does not guarantee that the induced accepted set
satisfies a desired  conditional exceedance rate. When an additional labeled
validation set is available, we can tune the acceptance
threshold $\lambda$ to match the target $\eta$.

Concretly, let $D_{\mathrm{val}}=\{(X_i,Y_i)\}_{i=1}^N$ be i.i.d.\ and independent of the data
used to construct $U_\alpha(\cdot)$, and define validation losses
$Z_i := L(g(X_i),Y_i)$.
Compute validation scores once,
\[
u_i := U_\alpha(X_i),\qquad i=1,\dots,N,
\]
and consider a finite grid of thresholds $\Lambda\subset\mathbb{R}$.
For each $\lambda\in\Lambda$, define the accepted index set and its size
\[
I_\lambda := \{i: u_i\le \lambda\},
\qquad
n_\lambda := |I_\lambda|,
\]
and estimate the conditional exceedance rate among accepted validation points by
\[
\widehat q(\lambda)
:=
\frac{1}{n_\lambda \,\vee\, 1}\sum_{i\in I_\lambda}\1\{Z_i>\tau\}.
\]
A simple selection rule is
\[
\widehat\lambda \in \arg\min_{\lambda\in\Lambda}\,|\widehat q(\lambda)-\eta|.
\]
In practice, it is often useful to additionally enforce a minimum acceptance
constraint (to avoid unstable estimates when $n_\lambda$ is too small), e.g.,
restricting to  
$\{n_\lambda/N\ge \rho_{\min}\}$.

This tuning step is heuristic. If one requires a fully distribution-free
validation-based choice of $\lambda$ with a high-probability guarantee that
$\Prbb(Z>\tau \mid X\in A_{\widehat\lambda;\alpha})\le \eta$ over a finite grid,
see Appendix~\ref{app:distribution-free}. A complementary heuristic that tunes
$\alpha$ while fixing $\lambda=\tau$ (\locusAlpha{}) is provided in the appendix
(Section~\ref{app:locus-alpha}).

\section{Experiments}

We test our methods on 13 datasets; see Appendix~\ref{app:exp-details} for a full description.
Each dataset is split into train (40\%), calibration (40\%), validation (10\%) and testing (10\%) portions. We first train a fixed  prediction function $g$ via a Random Forest   \citep{breiman2001random}. Using the calibration set, we fit a probabilistic model for
the realized loss $Z=L(g(X),Y)$ and apply the split calibration procedure from
Section~\ref{sec:method-U} to obtain the calibrated upper loss bound $U_\alpha(x)$.
We instantiate \locus{} with   the following predictive loss
distributions $\widetilde F(\cdot\mid x)$:
MC-dropout using a mixture density network \citep{bishop1994mdn} and
 Bayesian additive regression trees \citep{Chipman2010BART}. We also present the results for the epistemic-aware $\gamma(x)$-inflated variants of these.
For simplicity, for each dataset we set the tolerance level $\tau$ (which defines unacceptable error) so that 30\% of the validation observed losses exceed it, i.e., such that $\Prbb(Z>\tau) \approx 0.3$.
Implementation details are in Appendix~\ref{app:impl-details}.

\subsection{$U_\alpha$ and  Simple flagging}

First, we illustrate why the loss-calibrated score $U_\alpha$ is more interpretable for per-instance reliability than standard uncertainty proxies, and we also demonstrate the default \locus{} flagger with $\lambda=\tau$.
We focus on the \texttt{homes} dataset (King County house prices).
In this experiment we take $\tau=\$70{,}000$ as the operational threshold for an unacceptable absolute error in price, and we estimate $\widetilde F$ using MC Dropout.
As a reference uncertainty baseline, we also report $\widehat{sd}(Y\mid x)$, an estimate of the conditional standard deviation of the \emph{label} given the features. Details are in Section \ref{sec:locus_tuned} and Appendix~\ref{app:impl-details}.

Table~\ref{tab:sampled_examples} contrasts $U_\alpha$ with $\widehat{sd}(Y\mid x)$ and an outlier flag from Isolation Forest (baseline OOD detector; \citet{liu2008isolation}). The key distinction is interpretability on the \emph{loss scale}. While $\widehat{sd}(Y\mid x)$ measures dispersion of $Y\mid x$, it does not indicate how large the realized loss $Z=|g(x)-y|$ can be for that instance, since it does not account for the local quality of the fit of $g$.
For instance, the third row has a realized absolute error of $Z\approx\$187{,}830$, even though the estimated label variability is only $\widehat{sd}(Y\mid x)\approx\$17{,}150$; this illustrates that relatively low dispersion in $Y\mid X=x$ can still coincide with a very large loss when $g$ is locally misaligned. In the same row, the loss-quantile summaries make the risk scale explicit: $U_{0.5}(x)\approx\$38{,}640$ suggests a typical error on the order of tens of thousands of dollars, while $U_{0.1}(x)\approx\$188{,}920$ indicates that six-figure errors are plausible in the upper tail for that specific input.

\begin{table*}[!h]
\caption{\textbf{Example home-price predictions.}
Six test instances from house price predictions; all monetary quantities are in thousands of US dollars. We report $\hat y=g(x)$, $y$, and $Z=|\hat y-y|$, alongside \locus{} bounds $U_{0.5}(x)$ and $U_{0.1}(x)$ (approx.\ conditional median and 90th percentile of the loss at each $x$). The \locus{} $\alpha=10\%$ flag uses $\tau=\$70{,}000$ (flag iff $U_{0.1}(x)>\tau$); IF denotes an Isolation Forest outlier flag; $\widehat{sd}(Y\mid X=x)$ is a label-uncertainty baseline.   $U_\alpha(x)$ is interpretable in dollars and can indicate high-loss cases missed by variance or OOD-style scores.}
\centering
\small
\begin{tabular}{ccrrrrrr}
\toprule
\multicolumn{1}{c}{\makecell{\textbf{Flagged as} \\ \textbf{unreliable}\\ \textbf{by \locus{}?}}} &
\multicolumn{1}{c}{\makecell{\textbf{Flagged as}\\ \textbf{outlier}\\ \textbf{by IF?}}} &
\multicolumn{3}{c}{\makecell{\textbf{Outcomes}\\ \textbf{(thousands of US\$)}}} &
\multicolumn{2}{c}{\makecell{\textbf{\locus{}}\\ \textbf{(thousands of US\$)}}} &
\multicolumn{1}{c}{\textbf{Other}} \\
\cmidrule(lr){3-5}\cmidrule(lr){6-7}\cmidrule(lr){8-8}
 &  & Estimated price $\hat y$ & True price $y$ & $Z=|\hat y-y|$ & $U_{0.5}(x)$ & $U_{0.1}(x)$ & $\widehat{sd}(Y\mid X=x)$ \\
\midrule
yes & yes & 2358.64 & 3567.00 & 1208.36 & 294.06 & 1276.52 & 104.75 \\
yes & no & 1648.52 & 2200.00 & 551.48 & 203.00 & 554.59 & 104.07 \\
yes & yes & 337.83 & 150.00 & 187.83 & 38.64 & 188.92 & 17.15 \\
\addlinespace
no & yes & 217.85 & 215.00 & 2.85 & 17.94 & 67.93 & 76.95 \\
no & yes & 208.04 & 219.95 & 11.91 & 17.27 & 43.58 & 25.10 \\
no & no & 306.03 & 319.00 & 12.97 & 24.87 & 67.21 & 31.17 \\
\bottomrule
\end{tabular}
\label{tab:sampled_examples}
\end{table*}

In contrast, the second row shows the complementary failure mode for OOD-style screening: Isolation Forest does  not flag the point as an outlier, yet the realized error is $Z\approx\$551{,}480$ (predicting $\$1{,}648{,}520$ for a home that sold for $\$2{,}200{,}000$). Here, $U_{0.5}(x)\approx\$203{,}000$ and $U_{0.1}(x)\approx\$554{,}590$ again communicate directly—on the same dollar scale as the deployment tolerance—that errors well above $\tau=\$70{,}000$ are expected for this input. 
Finally, the bottom block highlights that \locus{} does not flag instances where losses are expected to be small with high probability. 

Next, we empirically assess whether the default flagging rule controls the rate of large losses among accepted predictions. Appendix Table~\ref{tab:condtail_locus_tau} reports that, with $\lambda=\tau$ and $\alpha=10\%$, \locus{} yields conditional exceedance probabilities that are close to the target across nearly all datasets: $\Prbb\!\bigl(Z>\tau \mid X\in A_{\tau;\alpha}\bigr)\ \lesssim\ \alpha.$ 
In addition, Appendix Table~\ref{tab:marginal_coverage_locus} verifies the marginal calibration property: the empirical coverage $\widehat{\Prbb}\!\bigl(Z\le U_\alpha(X)\bigr)$ is consistently at least $1-\alpha$, as expected from Theorem~\ref{thm:coverage}. For completeness, we also provide the \locus{} acceptance rate in the Appendix Table \ref{tab:acceptance_locus_tau}.

\begin{table*}[!h]
\centering
\caption{\textbf{Tuned conditional large-loss rate at matched acceptance.}
Conditional exceedance among accepted points, $\Prbb(Z>\tau \mid X\in A_\lambda)$ (\%).
For each dataset and method, the threshold $\lambda$ is selected on the validation split
to achieve an acceptance rate of $\approx 70\%$; test performance is then reported.
\locusLambda{} methods use the \locus{} score $U_\alpha$ with fixed calibration level $\alpha=10\%$.
Entries are the median over 30 runs, with the 5th and 95th percentiles in parentheses; lower is better.
Bold marks the lowest value in each row.}
\label{tab:condtail_tuned}
\scriptsize
\begin{tabular}{lcc cccc}
\toprule
& \multicolumn{2}{c}{\textbf{Baselines (no \locus{})}} & \multicolumn{4}{c}{\textbf{\locusLambda{} (loss-quantile score)}} \\
\cmidrule(lr){2-3}\cmidrule(lr){4-7}
\textbf{Dataset}
& \makecell{\textbf{IFlag}\\\footnotesize(IF outlier)}
& \makecell{\textbf{VARNet}\\\footnotesize(label var.)}
& \makecell{\textbf{\locus{}-BART}}
& \makecell{\textbf{\locus{}-BART}\\\footnotesize($\gamma$-infl.)}
& \makecell{\textbf{\locus{}-MC Dropout}}
& \makecell{\textbf{\locus{}-MC Dropout}\\\footnotesize($\gamma$-infl.)} \\
\midrule
airfoil & 31.0 (27.4; 35.3) & 27.5 (23.0; 30.7) & 24.7 (19.6; 28.0) & \textbf{23.1 (18.4; 28.2)} & 25.0 (18.6; 28.1) & 24.3 (19.4; 27.5) \\
bike & 30.1 (30.1; 30.1) & 28.4 (26.1; 29.2) & 20.1 (19.4; 21.6) & \textbf{20.0 (19.1; 20.8)} & 22.6 (21.8; 23.7) & 22.9 (21.7; 24.0) \\
concrete & 33.5 (27.2; 38.6) & 27.7 (23.0; 31.2) & \textbf{25.3 (18.7; 29.3)} & 26.2 (20.5; 30.4) & 27.7 (21.5; 31.1) & 27.4 (22.7; 32.4) \\
cycle & 28.9 (27.2; 30.0) & 29.7 (28.5; 30.2) & \textbf{27.5 (26.4; 28.6)} & 28.0 (26.6; 28.8) & 28.7 (26.8; 30.4) & 29.2 (27.9; 30.1) \\
electric & 31.5 (30.1; 32.1) & 29.4 (28.1; 30.7) & 24.2 (22.3; 25.4) & \textbf{24.0 (22.5; 26.2)} & 24.6 (22.3; 26.4) & 24.6 (22.9; 25.8) \\
homes & 34.8 (33.7; 36.0) & 26.6 (24.8; 27.7) & \textbf{18.8 (18.2; 20.2)} & 19.3 (17.5; 20.2) & 19.2 (17.9; 20.3) & 19.0 (17.9; 20.3) \\
meps19 & 36.4 (35.4; 37.4) & 24.8 (21.5; 26.6) & 15.7 (14.0; 17.4) & 15.7 (13.2; 17.4) & 14.3 (12.6; 15.2) & \textbf{14.0 (12.5; 15.1)} \\
protein & 28.0 (27.4; 29.0) & 27.4 (25.8; 29.4) & 23.9 (22.8; 24.9) & 23.9 (23.2; 25.0) & 23.6 (22.5; 24.3) & \textbf{23.5 (22.5; 24.4)} \\
star & 29.1 (26.3; 31.9) & 29.0 (26.9; 33.8) & 29.0 (26.3; 31.7) & 28.8 (26.4; 32.2) & \textbf{28.6 (24.3; 32.4)} & 29.3 (26.8; 32.5) \\
superconductivity & 24.0 (23.1; 25.3) & 25.8 (21.3; 28.1) & 18.1 (17.0; 19.7) & \textbf{18.0 (16.5; 19.5)} & 18.8 (17.7; 20.9) & 18.8 (17.8; 20.3) \\
wec & 34.8 (33.5; 35.3) & 23.0 (21.4; 24.6) & 15.2 (14.1; 16.6) & 14.9 (13.3; 16.0) & 8.9 (7.9; 10.3) & \textbf{8.8 (8.3; 9.8)} \\
winered & 30.2 (26.6; 32.9) & 28.8 (24.9; 32.9) & 26.7 (23.3; 30.8) & \textbf{26.4 (23.1; 29.7)} & 27.4 (22.8; 31.5) & 27.2 (22.0; 31.1) \\
winewhite & 29.9 (28.2; 31.5) & 30.3 (28.3; 31.7) & 27.4 (25.8; 29.5) & \textbf{26.8 (24.5; 29.1)} & 28.2 (25.5; 29.0) & 27.8 (24.9; 29.7) \\

\bottomrule
\end{tabular}
\end{table*}

\subsection{$\locusLambda$}
\label{sec:locus_tuned}

In this section, we compare \locusLambda{}---which fixes $\alpha$ (we use $\alpha=10\%$) and tunes the acceptance threshold $\lambda$ on a validation split---to two flagging baselines. (i) \textbf{IFlag}: an Isolation Forest anomaly score $O(x)$. (ii) \textbf{VARNet}: a estimate of label variance, $V(Y\mid X=x)$. 
Appendix \ref{app:baseline} presents implementation details for VARNet.

Let $S(x)$ be the score ($U_\alpha(x)$ for \locus{}, $\widehat V(Y\mid X=x)$ for VARNet, and $O(x)$ for IFlag), and define $A_\lambda := \{x: S(x)\le \lambda\}.$
For each dataset and method, we choose $\lambda$ so that $\widehat{\Prbb}(X\in A_\lambda)\approx 70\%$ on validation, then report $\Prbb(Z>\tau\mid X\in A_\lambda)$ on test. With acceptance matched, lower $\Prbb(Z>\tau\mid X\in A_\lambda)$ indicates better risk ranking.

Table \ref{tab:condtail_tuned} reports the median of 30 runs for $
\Prbb(Z>\tau\mid X\in A).
$ across all
13 datasets (with variability measurement in the form of quantiles $5\%$ and $95\%$ in parentheses).  The table shows that, at a matched acceptance rate of about $70\%$, the \locusLambda{} scores yield substantially lower conditional large-loss rates than both baselines on every dataset. Isolation-Forest flagging is consistently the weakest, and the label-variance baseline improves on it but remains well above the best loss-quantile scores.  Among \locusLambda{} variants, the winner alternates across datasets: BART (or its $\gamma$-inflated version) is strongest on several tabular benchmarks, while MC Dropout with $\gamma(x)$ inflation often performs best or ties for best.

Table \ref{tab:acceptance_locus_lambda} presents the empirical acceptance rates in the test split. We note that the $70\%$ acceptance rate restriction is approximately respected.

Alternatively, it's possible to tune the $\alpha$ parameter, which shows similar results to \locusLambda{}. Appendix \ref{app:locus-alpha} describes this procedure. Metrics $P(X\in A)$ and $P(Z>\tau | X\in A)$ are presented, respectively, in Appendix Tables \ref{tab:acceptance_locus_alpha} and \ref{tab:condtail_locus_alpha}.

\section{Final Remarks}

\locus{} outputs a single per-input quantity, the calibrated loss score
$U_\alpha(x)$, on the same scale as the deployed loss $Z=L(g(X),Y)$.
$U_\alpha$ is informative \emph{on its own}: it provides an interpretable,
decision-relevant risk signal for ranking, triage,  and resource allocation,
often without committing \emph{a priori} to a fixed accept/reject operating point.

By construction, $U_\alpha$ satisfies the  distribution-free  
guarantee $\Prbb\!\bigl(Z \le U_\alpha(X)\bigr)\ge 1-\alpha$ and, under mild   conditions on the
base loss model, $U_\alpha(x)$ converges to the conditional $(1-\alpha)$-quantile
of $Z\mid X=x$.

When a binary decision is needed, a transparent default  accepts
when $U_\alpha(x)\le\tau$ and guarantees explicit control of ``trusted-but-bad'' events:
$\Prbb\bigl(Z>\tau,\;U_\alpha(X)\le\tau\bigr)\le\alpha$ and, approximately,
 $\Prbb\bigl(Z>\tau \mid X\in A_{\tau;\alpha}\bigr)\;\lesssim\;\alpha$  (as the experiments corroborate).
 Alternatively, one may tune $\alpha$ or $\lambda$ to control errors inside the accepted set.

Our framework is modular: any engine that outputs a predictive CDF for the realized loss
$Z\mid X=x$ can be used in Step~2 (e.g., normalizing flows \citep{papamakarios2021normalizing},
FlexCode \citep{izbicki2017flexcode}, or mixture density networks \citep{bishop1994mdn} or others \citep{izbicki2016nonparametric,dalmasso2020conditional}),
while Step~3 conformalizes it into a valid upper bound. The optional $\gamma(x)$-inflated
envelope increases conservativeness in data-scarce regions without changing either the
calibration step or the finite-sample guarantees. Computationally, once the loss model
$\widetilde F(\cdot\mid x)$ is fit and the calibration PIT values are computed, varying
$\alpha$ only changes the scalar threshold $t_{1-\alpha}$; the corresponding family
$\{U_\alpha(x)\}_\alpha$ can therefore be obtained with minimal additional overhead, and
evaluating $U_\alpha(x)$ at new inputs is just a CDF inversion. 

Although our presentation focuses on regression, \locus{} only assumes a scalar loss is known. It therefore extends naturally to classification and other cost-sensitive problems as long as the task is governed by a single numerical notion of error.

% The code to implement Locus and reproduce results is available at \url{https://github.com/xxxxxx/xxxxxx} (to be disclose if the article is accepted). 

 \section*{Acknowledgements} % 
The work of M\'ario de Castro is partially funded by CNPq, Brazil (grant 301596/2025-5 ). Rafael Izbicki is grateful for the financial support of CNPq (422705/2021-7, 305065/2023-8 and 403458/2025-0) and FAPESP (grant 2023/07068-1).  Thiago R. Ramos is grateful for the financial support of CNPq (403458/2025-0). Denis Valle was partially supported by the US Department of Agriculture National Institute of Food and Agriculture McIntire–Stennis project 1005163 and US National Science Foundation award 2040819. 
  
% References
\newpage

\bibliography{bib2}

\newpage

\appendix
\title{LOCUS: A Distribution-Free Loss-Quantile Score for  Risk-Aware Predictions\\(Supplementary Material)}
\maketitle

\section{Experimental Details}
\label{app:exp-details}

\subsection{Datasets}
\label{app:datasets}

We use 13 standard regression datasets commonly used in the conformal prediction
literature: \texttt{airfoil} \citep{uci_2017}, \texttt{bike} \citep{bike_mehra_2023}, \texttt{concrete} \citep{uci_2017}, \texttt{cycle} \citep{uci_2017},
\texttt{homes} \citep{homes_kaggle_2016}, \texttt{electric} \citep{uci_2017}, \texttt{meps19} \citep{meps19_ahrq}, \texttt{protein} \citep{uci_2017},
\texttt{star} \citep{bike_mehra_2023}, \texttt{superconductivity} \citep{superconductivity_hamidieh_2018}, \texttt{wec} \citep{wec_neshat_2020}, \texttt{winered} \citep{wine_cortez_2009}, and
\texttt{winewhite} \citep{wine_cortez_2009}. Table~\ref{tab:datasets} summarizes the number of samples and
features.

\begin{table}[h]
\centering
\small
\begin{tabular}{lrr}
\toprule
Dataset & $n$ & $p$ \\
\midrule
Airfoil & 1503 & 5 \\
Bike & 10885 & 12 \\
Concrete & 1030 & 8 \\
Cycle & 9568 & 4 \\
Homes & 21613 & 17 \\
Electric & 10000 & 12 \\
Meps19 & 15781 & 141 \\
Protein & 45730 & 8 \\
Star & 2161 & 48 \\
Superconductivity & 21263 & 81 \\
WEC & 54000 & 49 \\
WineRed & 4898 & 11 \\
WineWhite & 1599 & 11 \\
\bottomrule
\end{tabular}
\caption{Dataset summary (sizes and feature counts).}
\label{tab:datasets}
\end{table}

\subsection{Splits and preprocessing}
\label{app:splits-preproc}

For each dataset, we randomly shuffle the full dataset and split it into
\textbf{40\% train}, \textbf{40\% calibration}, \textbf{10\% validation}, and
\textbf{10\% test}. The train split is used \emph{only} to fit the deployed
predictor $g$. The calibration split is used to build $U_\alpha(\cdot)$ via the
internal split $D_1/D_2$ (Section~\ref{sec:method-U}). The validation split is used
for the tuned variants (Appendix~\ref{sec:locus_tuned}). The test split is used only for
final reporting.

We standardize features and targets using \texttt{StandardScaler} fit on the
training split, then apply the same transformation to calibration/validation/test.
All reported losses are computed on this standardized scale.

\paragraph{Bike feature engineering.}
For \texttt{bike}, we construct time-derived covariates
(\texttt{hour}, \texttt{day}, \texttt{month}, \texttt{year}) from the timestamp and
drop the original \texttt{datetime}, \texttt{casual}, and \texttt{registered}
columns, matching the provided scripts.

\subsection{Deployed predictor $g$}
\label{app:predictor}

We use a Random Forest Regressor as the fixed deployed predictor $g$ on all
datasets, with \texttt{n\_estimators=300} and otherwise default settings. This predictor is trained only on
the training split.

\subsection{Loss and threshold $\tau$}
\label{app:loss-threshold}

We evaluate realized loss using absolute error on the standardized response:
$Z = |g(X)-Y|$.
Each run defines the unacceptable-loss threshold $\tau$ (denoted \texttt{T} in the
code) as the empirical $0.7$-quantile of the test losses under $g$.
This choice yields a comparable tail event across datasets (roughly
$\Prbb(Z>\tau)\approx 0.3$ by construction) and is used solely for benchmarking;
in deployment, $\tau$ would be set by domain requirements.

\subsection{Evaluation metrics}
\label{app:metrics}

Given a fitted bound $U_\alpha(\cdot)$ and an acceptance threshold $\lambda$, define
the acceptance region $A_{\lambda;\alpha}=\{x:U_\alpha(x)\le \lambda\}$. We report:

\begin{itemize}
  \item \textbf{Acceptance rate:}
  $p_A=\Prbb\big(X\in A_{\lambda;\alpha}\big)$.
  \item \textbf{Marginal tail rate:}
  $p_{\mathrm{big}\,Z}=\Prbb(Z>\tau)$, fixed as $30\%$ in all examples.
  \item \textbf{Conditional tail rate among accepted:}
  $p_{\mathrm{big}\,Z\mid A}=\Prbb\big(Z>\tau \mid X\in A_{\lambda;\alpha}\big)$.
  \item \textbf{Marginal coverage check:}
  $p_{\mathrm{conf}}=\Prbb\big(Z\le U_\alpha(X)\big)$ (empirically estimated).
\end{itemize}

\subsection{Repeated runs and uncertainty reporting}
\label{app:runs}

For each dataset, we run 30 repetitions with different random seeds.
Our tables report the average metric across runs, and we summarize variability
across runs in parentheses (via percentile intervals).

\section{Implementation Details for $\widetilde F(\cdot\mid x)$}
\label{app:impl-details}

This appendix documents the concrete choices used to instantiate the predictive
loss CDF $\widetilde F(\cdot\mid x)$ in the released code.

\subsection{Shared \locus{} construction}
\label{app:shared-locus}

Given calibration data $(X_i,Y_i)$, we compute $Z_i=|g(X_i)-Y_i|$ and randomly split
the calibration set into two equal halves, $D_1$ and $D_2$.
A probabilistic loss model is fit on $(X,Z)$ from $D_1$ to define a predictive loss
distribution, and $D_2$ is used to compute PIT values and the calibrated level
$t_{1-\alpha}$, yielding the final bound
$U_\alpha(x)=\widetilde F^{-1}(t_{1-\alpha}\mid x)$.

% \subsection{Baseline (locally weighted scale model)}
% \label{app:baseline}

% The baseline fits a Random Forest Regressor on $D_1$ to predict loss magnitude
% $\widehat m(x)\approx Z\mid X=x$. Calibration on $D_2$ is performed using the
% multiplicative score $u_i=Z_i/\widehat m(X_i)$ and its empirical $(1-\alpha)$
% quantile. The resulting bound is $U_\alpha(x)=t_{1-\alpha}\widehat m(x)$.

\subsection{MC Dropout via MDN}
\label{app:mdn}

We use a Mixture Density Network (MDN) to model $Z\mid X=x$ as a $K$-component
Gaussian mixture with $K=5$. The MDN has one hidden layer of width 64 and dropout
rate 0.4. To obtain an approximate posterior predictive distribution, we perform
MC dropout with 500 stochastic forward passes. For each $x$, we sample from the
resulting mixture to estimate quantiles and PIT values, enabling
$\widetilde F^{-1}(\cdot\mid x)$ and calibration.

\subsection{BART}
\label{app:bart}

We use a heteroskedastic BART model implemented in \texttt{pymc-bart} / \texttt{pymc}.
Concretely, we model (i) a mean function and (ii) a log-scale function using two
BART components with $m=100$ trees each, and a Gaussian likelihood for $Z$ given
these latent functions. Posterior inference uses four chains and 2000 MCMC draws
per chain. Posterior predictive samples of $Z\mid X=x$ are used to compute
quantiles and PIT values.

\subsection{Epistemic-aware $\gamma(x)$ inflation}
\label{app:gamma-impl}

For the BART and MC Dropout models, we modulate
conservativeness via a kNN density proxy computed on $D_1$ using $k=50$ nearest
neighbors in standardized feature space. The resulting kNN radius is mapped to a
$\gamma(x)\in(0,1)$ through a logistic transform; smaller $\gamma(x)$ in sparse
regions yields a more conservative predictive CDF envelope.

Operationally, for each test input $x$ we form a set of candidate CDF values across
posterior draws and take a lower envelope (a small quantile) to obtain an inflated
CDF $\widetilde F(\cdot\mid x)$. Inversion for $\widetilde F^{-1}$ is performed
numerically via a bisection procedure.

\section{Implementation details for VARNet}
\label{app:baseline}
In order to provide a estimative for $V(Y|X)$ we fit two neural networks with two 64 neuron layers and a single neuron output to predict $E(Y^2\mid X=x)$ and $E^2(Y\mid X=x)$. We then take
$$\widehat V(Y\mid X=x)=\widehat E(Y^2\mid X=x)-\widehat E^2(Y\mid X=x),$$

\noindent as a moment estimative of $V(Y|X=x)$. The PyTorch library was used in the implementation. The optimized loss function chosen is the mean absolute error (MAE).

\section{Alpha tuning with $\lambda=\tau$ (\locusAlpha{})}
\label{app:locus-alpha}

This appendix describes a simple heuristic for tuning the calibration level
$\alpha$ when one wishes to keep the operational threshold fixed at
$\lambda=\tau$, i.e., to accept whenever $U_\alpha(x)\le \tau$.
The motivation is to target a desired conditional exceedance level $\eta$ among
accepted points:
\[
\Prbb(Z>\tau \mid X\in A_{\tau;\alpha}) \approx \eta,
\qquad
A_{\tau;\alpha}=\{x:U_\alpha(x)\le \tau\}.
\]

Let $D_{\mathrm{val}}=\{(X_i,Y_i)\}_{i=1}^N$ be an additional labeled validation
set, i.i.d.\ and independent of the calibration data used to construct
$U_\alpha(\cdot)$. Define validation losses $Z_i := L(g(X_i),Y_i)$.
Fix a finite grid of candidate levels $\Gamma\subset(0,1)$.
For each $\alpha\in\Gamma$, compute the corresponding bound $U_\alpha(\cdot)$ and
validation scores
\[
u_i(\alpha) := U_\alpha(X_i),\qquad i=1,\dots,N,
\]
then define the accepted indices and their count
\[
I_\alpha := \{i: u_i(\alpha)\le \tau\},
\qquad
n_\alpha := |I_\alpha|.
\]
Estimate the conditional exceedance rate among accepted validation points by
\[
\widehat q(\alpha)
:=
\frac{1}{n_\alpha \,\vee\, 1}\sum_{i\in I_\alpha}\1\{Z_i>\tau\}.
\]
A basic selection rule is
\[
\widehat\alpha \in \arg\min_{\alpha\in\Gamma}\,|\widehat q(\alpha)-\eta|.
\]
As with $\lambda$-tuning, it is often useful to impose a minimum acceptance
constraint, e.g., restrict to $\{\alpha\in\Gamma: n_\alpha/N\ge \rho_{\min}\}$.

\paragraph{Computation note.}
Once $\widetilde F(\cdot\mid x)$ is fit on $D_1$, the PIT values
$W_i=\widetilde F(Z_i\mid X_i)$ on $D_2$ can be computed once; changing $\alpha$
only changes the empirical quantile level $t_{1-\alpha}$ and therefore the final
bound $U_\alpha(x)=\widetilde F^{-1}(t_{1-\alpha}\mid x)$.
Thus, scanning $\alpha$ mainly amounts to reusing the same PIT values and
recomputing the threshold $t_{1-\alpha}$.

\paragraph{Guarantee.}
For any \emph{fixed} (pre-specified) $\alpha$, \locusT{} with $\lambda=\tau$
retains the distribution-free joint tail guarantee
$\Prbb(Z>\tau,\ X\in A_{\tau;\alpha})\le \alpha$ (Theorem~\ref{thm:loss-control-A}).
The data-driven choice $\widehat\alpha$ above is a practical heuristic and is
reported as such.

\section{Distribution-free alternatives to \locusLambda}
\label{app:distribution-free}
To complement the heuristic tuning of $\lambda$ in \locusLambda{}, we provide
a fully distribution-free way to choose $\lambda$ using the validation sample.
The goal is to guarantee, with high probability over the randomness in the
calibration and validation data, that the conditional exceedance rate among
accepted predictions does not exceed a target $\eta$:
\[
\Prbb\!\bigl(Z>\tau \mid X\in A_{\lambda;\alpha}\bigr)\le \eta,
\]
where $A_{\lambda;\alpha}:=\{x:U_\alpha(x)\le \lambda\}.$
We consider a data-driven choice $\lambda^\star$ obtained by scanning a finite
grid $\Lambda$ and enforcing the constraint via a simultaneous upper confidence
certificate for the ratio
\[
q(\lambda)
=
\Prbb(Z>\tau \mid X\in A_{\lambda;\alpha})
=
\frac{H(\lambda)}{G(\lambda)},
\]
where
\[
G(\lambda):=\Prbb(X\in A_{\lambda;\alpha}),
\qquad
H(\lambda):=\Prbb(Z>\tau,\ X\in A_{\lambda;\alpha}).
\]
Using uniform finite-sample deviation bounds for the empirical numerator and
denominator over $\Lambda$, we construct an empirical upper certificate of the form
\[
\overline q(\lambda)
:=
\frac{\widehat H_N(\lambda)+\varepsilon_H}{\widehat G_N(\lambda)-\varepsilon_G},
\]
and define $\lambda^\star$ as the largest threshold in $\Lambda$ satisfying
$\overline q(\lambda)\le \eta$ (with the additional requirement
$\widehat G_N(\lambda)>\varepsilon_G$ so that the denominator is positive).
If no such threshold exists, we use the convention that
$A_{\lambda^\star;\alpha}=\varnothing$.

The next theorem shows that this choice of $\lambda^\star$ yields a
distribution-free, finite-sample guarantee: with probability at least
$1-\delta$ over $(D,D_{\mathrm{val}})$, the conditional exceedance probability
among accepted points is at most $\eta$.

\begin{theorem}
\label{thm:uniform-lambda-explicit-correct}
Fix $\alpha\in(0,1)$ and let $U_\alpha$ be constructed
from an independent calibration dataset $D$. Let
$D_{\mathrm{val}}=\{(X_i,Y_i)\}_{i=1}^N$ be an additional i.i.d.\ validation
sample, independent of $D$, and write $Z_i=L(g(X_i),Y_i)$ and $u_i=U_\alpha(X_i)$.
Fix $\tau\in\mathbb{R}$ and a finite set of candidate thresholds $\Lambda$.

For $\lambda\in\Lambda$, let
\[
G(\lambda):=\Prbb(X\in A_{\lambda;\alpha})=\Prbb(u\le \lambda),
\qquad
H(\lambda):=\Prbb(Z>\tau,\ X\in A_{\lambda;\alpha})=\Prbb(Z>\tau,\ u\le \lambda),
\]
and, whenever $G(\lambda)>0$,
\[
q(\lambda):=\Prbb(Z>\tau\mid X\in A_{\lambda;\alpha})=\frac{H(\lambda)}{G(\lambda)}.
\]
Also let
\[
\widehat G_N(\lambda):=\frac1N\sum_{i=1}^N \1\{u_i\le \lambda\},
\qquad
\widehat H_N(\lambda):=\frac1N\sum_{i=1}^N \1\{u_i\le \lambda\}\1\{Z_i>\tau\}.
\]

Fix $\delta\in(0,1)$ and $\eta\in(0,1)$, and set
\[
\varepsilon_H
:=
2\sqrt{\frac{\log(2(N+1))}{N}}
+\sqrt{\frac{\log(4/\delta)}{2N}},
\qquad
\varepsilon_G
:=
\sqrt{\frac{\log(4/\delta)}{2N}}.
\]
For $\lambda\in\Lambda$ with $\widehat G_N(\lambda)>\varepsilon_G$, set
\[
\overline q(\lambda)
:=
\frac{\widehat H_N(\lambda)+\varepsilon_H}{\widehat G_N(\lambda)-\varepsilon_G}.
\]
Let
\[
\mathcal F
:=
\Bigl\{\lambda\in\Lambda:\ \widehat G_N(\lambda)>\varepsilon_G
\ \text{and}\ \overline q(\lambda)\le \eta\Bigr\},
\]
and choose
$
\lambda^\star:=\max\mathcal F,
$
with the convention that if $\mathcal F=\varnothing$, then $\lambda^\star$ is chosen
so that $A_{\lambda^\star;\alpha}=\varnothing$.

Then, with probability at least $1-\delta$ over the randomness in $(D,D_{\mathrm{val}})$,
\[
\Prbb\!\bigl(Z>\tau \mid X\in A_{\lambda^\star;\alpha}\bigr)\le \eta,
\]
where the inner probability is with respect to an independent test draw $(X,Y)$
from the deployment distribution, and we adopt the convention that the left-hand
side equals $0$ if $\Prbb(X\in A_{\lambda^\star;\alpha})=0$.
\end{theorem}

\section{Proofs and Additional Theorems}
\label{app:proofs}

\begin{proof}[Proof of Theorem \ref{thm:coverage}]
Let $n_2:=|I_2|$ and define, for an independent test pair $(X,Y)$ with
$Z=L(g(X),Y)$,
\[
W_{n_2+1}:=\widetilde F(Z\mid X),
\qquad
W_i:=\widetilde F(Z_i\mid X_i)\ \ (i\in I_2).
\]
Conditional on $D_1$, the variables $W_1,\dots,W_{n_2},W_{n_2+1}$ are i.i.d.
(hence exchangeable). Let $W_{(1)}\le\cdots\le W_{(n_2)}$ be the order
statistics of $\{W_i:i\in I_2\}$ and let
\[
k:=\left\lceil (1-\alpha)(n_2+1)\right\rceil,\qquad
t_{1-\alpha}:=
\begin{cases}
W_{(k)}, & k\le n_2,\\
1, & k=n_2+1,
\end{cases}
\]
which is the usual split-conformal empirical $(1-\alpha)$-quantile. By
exchangeability,
\[
\Prbb\!\bigl(W_{n_2+1}\le t_{1-\alpha}\mid D_1\bigr)\ge \frac{k}{n_2+1}\ge 1-\alpha.
\]
Since for each $x$ the map $z\mapsto \widetilde F(z\mid x)$ is
nondecreasing,
$W_{n_2+1}\le t_{1-\alpha}$ is equivalent to
\[
Z\le \widetilde F^{-1}(t_{1-\alpha}\mid X)=U_\alpha(X).
\]
Taking expectations over $D_1$ yields $\Prbb(Z\le U_\alpha(X))\ge 1-\alpha$.

Now, to prove the second part of the theorem,
let the rank of $W_{n_2+1}$ among $W_1,\dots,W_{n_2},W_{n_2+1}$ be $R \;:=\; 1+\sum_{i\in I_2}\1\{W_i < W_{n_2+1}\}.$  By exchangeability and the no-ties assumption, $R$ is uniform on
$\{1,\dots,n_2+1\}$, so
\[
\Prbb(R\le k\mid D_1)=\frac{k}{n_2+1}.
\]
Moreover, when $k\le n_2$ we have $t_{1-\alpha}=W_{(k)}$ and no ties imply
\[
\{W_{n_2+1}\le t_{1-\alpha}\}=\{W_{n_2+1}\le W_{(k)}\}=\{R\le k\},
\]
while if $k=n_2+1$ then $t_{1-\alpha}=1$ and the same identity holds trivially.
Hence, in all cases,
\[
\Prbb(W_{n_2+1}\le t_{1-\alpha}\mid D_1)=\frac{k}{n_2+1}.
\]
Using again monotonicity of $z\mapsto \widetilde F(z\mid x)$, this yields
\[
\Prbb\!\bigl(Z\le U_\alpha(X)\mid D_1\bigr)=\frac{k}{n_2+1}.
\]
Finally, since $k=\lceil(1-\alpha)(n_2+1)\rceil$, we have
$k<(1-\alpha)(n_2+1)+1$, and therefore
\[
\Prbb\!\bigl(Z\le U_\alpha(X)\mid D_1\bigr)
=\frac{k}{n_2+1}
< 1-\alpha+\frac{1}{n_2+1}.
\]
Taking expectations over $D_1$ gives
$\Prbb(Z\le U_\alpha(X))<1-\alpha+\frac{1}{n_2+1}$, proving the claimed upper
bound.
\end{proof}

\begin{proof}[Proof of Theorem \ref{thm:loss-control-A}]
On $\{X\in A_{\tau;\alpha}\}$ we have $U_\alpha(X)\le \tau$, hence
\[
\{Z>\tau,\;X\in A_{\tau;\alpha}\}\subseteq \{Z>U_\alpha(X)\}.
\]
By Theorem~\ref{thm:coverage}, $\Prbb(Z>U_\alpha(X))\le \alpha$, proving
$\Prbb(Z>\tau,\;X\in A_{\tau;\alpha})\le \alpha$. 
\end{proof}

\begin{assumption}[Conditions for asymptotic conditional coverage]\label{ass:acc}
Consider a sequence indexed by the total calibration size $n=|D|$, with split
sizes $n_1:=|D_1|\to\infty$ and $n_2:=|D_2|\to\infty$ as $n\to\infty$.
Let $\widetilde F_{n_1}(\cdot\mid x)$ denote the predictive CDF constructed from
$D_1$ (this is $\widetilde F(\cdot\mid x)$ in the main text, with
dependence on $n_1$ made explicit). Assume:

\begin{enumerate}[label=(A\arabic*), leftmargin=*]
\item \textbf{Random-design uniform consistency (in probability).}
For an independent deployment draw $(X,Z)$ (independent of $D_1$), define
\[
\Delta_{n_1}
:=
\sup_{z\in\mathbb R}
\bigl|\widetilde F_{n_1}(z\mid X)-F_{Z\mid X}(z\mid X)\bigr|.
\]
Then,
\[
\Delta_{n_1}\xrightarrow{\Prbb}0
\qquad (n_1\to\infty),
\]
where $\xrightarrow{\Prbb}$ denotes convergence in probability with respect to
the randomness in $D_1$ and the independent draw $X$.

\item \textbf{Pointwise uniform consistency at the target $x$ (in probability).}
For the fixed $x\in\mathcal X$ under consideration,
\[
\sup_{z\in\mathbb R}
\bigl|\widetilde F_{n_1}(z\mid x)-F_{Z\mid X}(z\mid x)\bigr|
\xrightarrow{\Prbb}0
\qquad (n_1\to\infty),
\]
where the probability is over the randomness in $D_1$.

\item \textbf{Regularity of the true conditional law and at the target $x$.}
(i) For $P_X$-almost every $x'$, the true conditional CDF $F_{Z\mid X}(\cdot\mid x')$
is continuous.
(ii) At the fixed $x\in\mathcal X$ under consideration, $F_{Z\mid X}(\cdot\mid x)$ is
continuous and its $(1-\alpha)$-quantile $
q_{1-\alpha}(x):= F_{Z\mid X}^{-1}(1-\alpha\mid x)
$ is unique.

\end{enumerate}
\end{assumption}

\begin{proof}[Proof of Theorem \ref{thm:asymp-coverage}]
Fix $x\in\mathcal X$ as in the theorem. We prove the result in four steps.

\smallskip
\noindent\textbf{Step 1: PIT values are asymptotically uniform, and their conditional CDF converges to the identity.}
Let $(X,Z)$ be an independent deployment draw, independent of $D_1$. Define the
oracle PIT value
\[
V:=F_{Z\mid X}(Z\mid X).
\]
Under Assumption~\ref{ass:acc}(A3)(i), for $P_X$-almost every $x'$ the conditional CDF
$F_{Z\mid X}(\cdot\mid x')$ is continuous, hence the probability integral transform yields
\[
V\sim \mathrm{Unif}(0,1).
\]
Define the estimated PIT value
\[
W:=\widetilde F_{n_1}(Z\mid X).
\]
By Assumption~\ref{ass:acc}(A1),
\[
|W-V|
=\bigl|\widetilde F_{n_1}(Z\mid X)-F_{Z\mid X}(Z\mid X)\bigr|
\le
\sup_{z\in\mathbb R}\bigl|\widetilde F_{n_1}(z\mid X)-F_{Z\mid X}(z\mid X)\bigr|
=\Delta_{n_1}
\xrightarrow{\Prbb}0.
\]
Hence $W-V\xrightarrow{\Prbb}0$, and since $V\sim\mathrm{Unif}(0,1)$, Slutsky's theorem implies
\[
W \xrightarrow{Dist} \mathrm{Unif}(0,1).
\]

We also need a uniform control of the \emph{conditional} CDF of $W$ given $D_1$.
Let
\[
H_{n_1}(w):=\Prbb(W\le w\mid D_1), \qquad w\in[0,1].
\]
Fix $\delta\in(0,1)$. Using the inclusions
\[
\{V\le w-\delta\}\subseteq \{W\le w\}\cup\{|W-V|>\delta\},
\qquad
\{W\le w\}\subseteq \{V\le w+\delta\}\cup\{|W-V|>\delta\},
\]
we obtain, for all $w\in[0,1]$,
\[
w-\delta-\Prbb(|W-V|>\delta\mid D_1)
\le H_{n_1}(w)
\le w+\delta+\Prbb(|W-V|>\delta\mid D_1).
\]
Therefore,
\begin{equation}\label{eq:step1-sup-bound}
\sup_{w\in[0,1]}|H_{n_1}(w)-w|
\le \delta+\Prbb(|W-V|>\delta\mid D_1).
\end{equation}

It remains to show that $\Prbb(|W-V|>\delta\mid D_1)\xrightarrow{\Prbb}0$.
Since $|W-V|\le \Delta_{n_1}$, we have
\[
\Prbb(|W-V|>\delta\mid D_1)\le \Prbb(\Delta_{n_1}>\delta\mid D_1).
\]
Moreover, Assumption~\ref{ass:acc}(A1) implies $\Prbb(\Delta_{n_1}>\delta)\to 0$.
Then for any $\varepsilon>0$, Markov's inequality yields
\[
\Prbb\!\left(\Prbb(\Delta_{n_1}>\delta\mid D_1)>\varepsilon\right)
\le
\frac{\E\big[\Prbb(\Delta_{n_1}>\delta\mid D_1)\big]}{\varepsilon}
=
\frac{\Prbb(\Delta_{n_1}>\delta)}{\varepsilon}
\to 0,
\]
so $\Prbb(\Delta_{n_1}>\delta\mid D_1)\xrightarrow{\Prbb}0$, and hence also
$\Prbb(|W-V|>\delta\mid D_1)\xrightarrow{\Prbb}0$.

Plugging this into \eqref{eq:step1-sup-bound} and then letting $\delta\downarrow 0$ gives
\begin{equation}\label{eq:Hn1-to-id}
\sup_{w\in[0,1]}|H_{n_1}(w)-w|\xrightarrow{\Prbb}0 .
\end{equation}

\smallskip
\noindent\textbf{Step 2: The empirical PIT quantile converges to $1-\alpha$.}
Conditional on $D_1$, the points in $D_2$ are i.i.d.\ and independent of $D_1$,
so $\{W_{i,n}\}_{i\in I_2}$ are i.i.d.\ with CDF $H_{n_1}$.
By the Glivenko--Cantelli theorem (applied conditionally on $D_1$),
\[
\sup_{w\in[0,1]}|\widehat H_{n_2}(w)-H_{n_1}(w)|
\xrightarrow{\text{a.s.}} 0
\qquad(n_2\to\infty),
\]
almost surely in the randomness of $D_2$ conditional on $D_1$.
Combining with \eqref{eq:Hn1-to-id} and the triangle inequality gives
\[
\sup_{w\in[0,1]}|\widehat H_{n_2}(w)-w|
\xrightarrow{\Prbb}0.
\]

We now use continuity of the quantile functional at the uniform CDF. Since the
limit CDF $w\mapsto w$ is continuous and strictly increasing, uniform convergence
of CDFs implies convergence of generalized inverses, hence
\[
t_{1-\alpha,n}=\widehat H_{n_2}^{-1}(1-\alpha)\xrightarrow{\Prbb}1-\alpha.
\]

\smallskip
\noindent\textbf{Step 3: $U_{\alpha,n}(x)$ converges to the oracle quantile.}
Let $q_{1-\alpha}(x):=F_{Z\mid X}^{-1}(1-\alpha\mid x)$, unique by
Assumption~\ref{ass:acc}(A3).
By Assumption~\ref{ass:acc}(A2),
\[
\sup_{z\in\mathbb R}\bigl|\widetilde F_{n_1}(z\mid x)-F_{Z\mid X}(z\mid x)\bigr|
\xrightarrow{\Prbb}0,
\]
and Step~2 gives $t_{1-\alpha,n}\xrightarrow{\Prbb}1-\alpha$.
By the same quantile-map continuity argument used in Step~2 (here using the
strict increase/continuity at the target quantile from Assumption~\ref{ass:acc}(A3)),
it follows that
\[
U_{\alpha,n}(x)
=\widetilde F_{n_1}^{-1}(t_{1-\alpha,n}\mid x)
\xrightarrow{\Prbb}
F_{Z\mid X}^{-1}(1-\alpha\mid x)
=q_{1-\alpha}(x).
\]

\smallskip
\noindent\textbf{Step 4: Conditional coverage convergence.}
Because $D$ is independent of the test point, conditioning on $D$ and $X=x$ gives
\[
\Prbb\!\bigl(Z\le U_{\alpha,n}(x)\mid X=x,\; D\bigr)
=
F_{Z\mid X}\bigl(U_{\alpha,n}(x)\mid x\bigr).
\]
By Step~3, $U_{\alpha,n}(x)\xrightarrow{\Prbb}q_{1-\alpha}(x)$.
Since $u\mapsto F_{Z\mid X}(u\mid x)$ is continuous at $q_{1-\alpha}(x)$
(Assumption~\ref{ass:acc}(A3)), the continuous mapping theorem yields
\[
F_{Z\mid X}\bigl(U_{\alpha,n}(x)\mid x\bigr)
\xrightarrow{\Prbb}
F_{Z\mid X}\bigl(q_{1-\alpha}(x)\mid x\bigr)
=1-\alpha.
\]
This proves the first displayed convergence in the theorem.

\end{proof}

Finally, we prove Theorem~\ref{thm:uniform-lambda-explicit-correct}.

\begin{proof}[Proof of Theorem~\ref{thm:uniform-lambda-explicit-correct}]

Let
\[
G(\lambda):=\Prbb(X\in A_{\lambda;\alpha})=\Prbb(u\le \lambda),\qquad
H(\lambda):=\Prbb(Z>\tau,\ X\in A_{\lambda;\alpha})=\Prbb(Z>\tau,\ u\le \lambda),
\]
so that for $G(\lambda)>0$,
\[
q(\lambda):=\Prbb(Z>\tau\mid X\in A_{\lambda;\alpha})=\frac{H(\lambda)}{G(\lambda)}.
\]

\paragraph{Step 1: VC-type uniform deviation for the numerator.}
Define
\[
\Delta_H:=\sup_{\lambda\in\Lambda}\bigl|\widehat H_N(\lambda)-H(\lambda)\bigr|.
\]
Since changing one observation $(X_i,Z_i)$ changes each summand $A_i^\lambda B_i\in\{0,1\}$ by at most $1$,
it changes $\widehat H_N(\lambda)$ by at most $1/N$ for every $\lambda$, hence changes $\Delta_H$ by at most $1/N$.
By McDiarmid's inequality \cite[Theorem~2.2]{devroye2001combinatorial}, for all $\varepsilon>0$,
\[
\Prbb\Big(\big|\Delta_H-\E\Delta_H\big|>\varepsilon\Big)\le 2\exp(-2N\varepsilon^2).
\]
To bound $\E\Delta_H$, consider the class of sets
\[
\mathcal{C}:=\Big\{(x,z): U_\alpha(x)\le \lambda,\ z>\tau\ \Big|\ \lambda\in\Lambda\Big\}.
\]
Then $\Delta_H=\sup_{C\in\mathcal{C}}|P_N(C)-P(C)|$, and by the maximal inequality
\cite[Theorem~3.1]{devroye2001combinatorial},
\[
\E\Delta_H \le 2\sqrt{\frac{\log\big(2\mathcal{S}(N)\big)}{N}},
\]
where $\mathcal{S}(N)$ is the shatter coefficient of $\mathcal{C}$.
Fix any sample $\{(x_i,z_i)\}_{i=1}^N$ and write $u_i:=U_\alpha(x_i)$. Since $\tau$ is fixed,
\[
\1\{(x_i,z_i)\in\mathcal{C}_\lambda\}=\1\{u_i\le \lambda\}\1\{z_i>\tau\}.
\]
Points with $z_i\le\tau$ are always labeled $0$, and among points with $z_i>\tau$ the labeling is determined by a
single threshold $\lambda$ on $\{u_i\}$. Hence, as $\lambda$ varies, the labeling can change only when $\lambda$
crosses one of the at most $N$ values $\{u_i\}_{i=1}^N$, so at most $N+1$ distinct labelings are possible.
Therefore $\mathcal{S}(N)\le N+1$, and
\[
\E\Delta_H \le 2\sqrt{\frac{\log\big(2(N+1)\big)}{N}}.
\]
Combining with McDiarmid and taking $\delta_H\in(0,1)$ yields that with probability at least $1-\delta_H$,
\begin{equation}\label{eq:EHbound}
\sup_{\lambda\in\Lambda}\bigl|\widehat H_N(\lambda)-H(\lambda)\bigr|
\le
\varepsilon_H
:=
2\sqrt{\frac{\log\big(2(N+1)\big)}{N}}
+\sqrt{\frac{\log(2/\delta_H)}{2N}}.
\end{equation}

\begin{remark}[Sharper uniform deviations]
Using more refined empirical process techniques (e.g., chaining/entropy bounds),
one can obtain a DKW-type uniform deviation bound of order
\[
\sup_{\lambda\in\Lambda}\bigl|\widehat H_N(\lambda)-H(\lambda)\bigr|
\;\lesssim\;
\sqrt{\frac{\log(1/\delta_H)}{N}}
\]
(up to universal constants), removing the extra $\sqrt{\log N}$ factor; see, e.g.,
\cite[Section~3.2]{devroye2001combinatorial}. We keep \eqref{eq:EHbound} for simplicity.
\end{remark}

\paragraph{Step 2: DKW uniform deviation for the denominator.}
Define $\Delta_G:=\sup_{\lambda\in\Lambda}|\widehat G_N(\lambda)-G(\lambda)|$.
Since $\widehat G_N(\lambda)$ is the empirical CDF of the scalar variables $u_i=U_\alpha(X_i)$,
the Dvoretzky--Kiefer--Wolfowitz \citep[Theorem 3.3]{devroye2001combinatorial} inequality implies that for any $\delta_G\in(0,1)$, with probability at least $1-\delta_G$,
\begin{equation}\label{eq:DKWbound}
\sup_{\lambda\in\Lambda}\bigl|\widehat G_N(\lambda)-G(\lambda)\bigr|
\le
\varepsilon_G
:=
\sqrt{\frac{\log(2/\delta_G)}{2N}}.
\end{equation}
% In particular, on this event, $G(\lambda)\ge \widehat G_N(\lambda)-\varepsilon_G$ for all $\lambda$.

\paragraph{Step 3: Uniform ratio bound and choice of $\lambda^\star$.}
Let $\mathcal E_H$ denote the event in \eqref{eq:EHbound} and $\mathcal E_G$ the event in \eqref{eq:DKWbound}, i.e.,
\[
\mathcal E_H:=\Big\{\sup_{\lambda\in\Lambda}\bigl|\widehat H_N(\lambda)-H(\lambda)\bigr|\le \varepsilon_H\Big\},
\qquad
\mathcal E_G:=\Big\{\sup_{\lambda\in\Lambda}\bigl|\widehat G_N(\lambda)-G(\lambda)\bigr|\le \varepsilon_G\Big\}.
\]
Define their intersection
\[
\mathcal E:=\mathcal E_H\cap \mathcal E_G
=
\Big\{\sup_{\lambda\in\Lambda}\bigl|\widehat H_N(\lambda)-H(\lambda)\bigr|\le \varepsilon_H,\ 
\sup_{\lambda\in\Lambda}\bigl|\widehat G_N(\lambda)-G(\lambda)\bigr|\le \varepsilon_G\Big\}.
\]
By \eqref{eq:EHbound} and \eqref{eq:DKWbound} and a union bound (taking $\delta_H=\delta_G=\delta/2$),
\[
\Prbb(\mathcal E)\ \ge\ 1-\delta.
\]

For $\lambda\in\Lambda$ with $\widehat G_N(\lambda)>\varepsilon_G$, define the empirical ratio
\[
\overline q(\lambda)
:=
\frac{\widehat H_N(\lambda)+\varepsilon_H}{\widehat G_N(\lambda)-\varepsilon_G}.
\]
We now choose $\lambda$ by enforcing $\overline q(\lambda)\le \eta$.
Define the feasible set of thresholds
\[
\mathcal F
:=
\Big\{\lambda\in\Lambda:\ \widehat G_N(\lambda)>\varepsilon_G\ \text{ and }\ \overline q(\lambda)\le \eta\Big\},
\]
and let $\lambda^\star:=\max\mathcal F$ (with the convention that if $\mathcal F=\varnothing$ then
$A_{\lambda^\star;\alpha}=\varnothing$). By construction, whenever $\mathcal F\neq\varnothing$ the selected
$\lambda^\star$ satisfies $\widehat G_N(\lambda^\star)>\varepsilon_G$.

On the event $\mathcal E=\mathcal E_H\cap \mathcal E_G$, we have for all $\lambda\in\Lambda$ that
$H(\lambda)\le \widehat H_N(\lambda)+\varepsilon_H$ and $G(\lambda)\ge \widehat G_N(\lambda)-\varepsilon_G$.
Applying this at $\lambda=\lambda^\star$ and using $\widehat G_N(\lambda^\star)>\varepsilon_G$ yields
\[
G(\lambda^\star)\ \ge\ \widehat G_N(\lambda^\star)-\varepsilon_G\ >\ 0,
\]
so $q(\lambda^\star)=H(\lambda^\star)/G(\lambda^\star)$ is well-defined, and moreover
\[
q(\lambda^\star)
=\frac{H(\lambda^\star)}{G(\lambda^\star)}
\le
\frac{\widehat H_N(\lambda^\star)+\varepsilon_H}{\widehat G_N(\lambda^\star)-\varepsilon_G}
=
\overline q(\lambda^\star)
\le \eta,
\]
where the last inequality holds by the definition of $\lambda^\star\in\mathcal F$.

Finally, since $\mathcal E\subseteq\{q(\lambda^\star)\le \eta\}$, we conclude that
\[
\Prbb\!\Bigl(q(\lambda^\star)\le \eta\Bigr)\ \ge\ 1-\delta.
\]

\end{proof}

\begin{theorem}
\label{thm:selection_tail_control}
Let $(X,Z)$ be random variables and let $U_\alpha:\mathcal X\to\mathbb R$ be measurable. 
Fix $\tau\in\mathbb R$ and define the selection event
$
A \;=\;\{\,U_\alpha(X) < \tau\,\},$
and assume $\mathbb P(A)>0$. If $U_\alpha$ is conditionally valid at level $\alpha$, i.e.
\begin{equation}
\label{eq:cond_valid}
\mathbb P\!\bigl(Z > U_\alpha(X)\mid X\bigr)\le \alpha
\quad \text{a.s.},
\end{equation}
then $\mathbb P(Z>\tau \mid A)\le \alpha.$ 
\end{theorem}

\begin{proof}
On the event $A$ we have $U_\alpha(X)<\tau$, hence $\{Z>\tau\}\subseteq \{Z>U_\alpha(X)\}$.
Therefore,
\[
\mathbb P(Z>\tau\mid X)\mathbf 1_A
\;\le\;
\mathbb P(Z>U_\alpha(X)\mid X)\mathbf 1_A
\;\le\; \alpha \mathbf 1_A,
\]
where the last inequality uses \eqref{eq:cond_valid}. Taking conditional expectation given $A$ yields
\[
\mathbb P(Z>\tau\mid A)=\mathbb E\big[\mathbb P(Z>\tau\mid X)\mid A\big]\le \alpha.
\]
\end{proof}

\newpage

\section{Other tables}
Here we present the simulation results for the remaining variants of \locus{}.

\begin{table}[h]
\caption{\textbf{Marginal coverage check for \locus{}.} Reports the coverage probability, $\Prbb (Z < U_\alpha (x))$, which is expected to be approximately $1-\alpha$. }
\begin{tabular}{lllll}
\toprule
\textbf{Dataset} & \textbf{BART} (\%) & \textbf{BART ($\gamma$)} (\%) & \textbf{MC Dropout} (\%) & \textbf{MC Dropout($\gamma$)} (\%) \\
\midrule
airfoil & 89.4 (80.4; 93.1) & 90.7 (84.1; 95.4) & 88.7 (81.1; 93.4) & {91.1 (80.9; 94.8)} \\
bike & 90.6 (89.6; 91.4) & 91.3 (90.0; 91.9) & 91.1 (90.5; 92.2) & {92.8 (92.0; 93.6)} \\
concrete & 90.8 (84.5; 94.7) & 91.3 (84.4; 96.1) & 89.3 (84.9; 95.2) & {91.7 (83.0; 97.1)} \\
cycle & 90.3 (88.5; 91.6) & 90.9 (89.4; 92.1) & 90.0 (88.5; 91.8) & {91.1 (89.3; 92.3)} \\
electric & 90.3 (88.8; 92.0) & 90.5 (89.2; 92.4) & 90.0 (88.1; 91.3) & {91.3 (89.8; 92.7)} \\
homes & 90.5 (89.0; 91.3) & 90.9 (90.0; 91.9) & 90.3 (89.2; 91.1) & {92.4 (91.1; 93.5)} \\
meps19 & 89.7 (88.4; 91.1) & {92.0 (90.3; 93.0)} & 89.9 (88.2; 91.3) & 90.9 (89.8; 92.8) \\
protein & 89.9 (89.1; 90.6) & 90.6 (89.9; 91.6) & 89.7 (89.0; 90.7) & {91.6 (90.8; 92.2)} \\
star & 90.1 (85.9; 92.4) & {91.0 (86.4; 95.0)} & 90.8 (86.1; 93.1) & 90.3 (87.1; 94.1) \\
superconductivity & 89.8 (88.3; 90.8) & 90.6 (88.8; 91.5) & 89.9 (87.8; 90.9) & {91.6 (90.1; 92.6)} \\
wec & 89.9 (89.3; 90.5) & 90.8 (90.0; 91.2) & 89.8 (89.4; 90.3) & {92.3 (91.6; 93.1)} \\
winered & 90.3 (86.5; 95.1) & {92.2 (86.5; 96.0)} & 90.0 (84.9; 95.6) & 91.2 (87.2; 96.0) \\
winewhite & 90.2 (88.2; 92.8) & 91.1 (89.5; 93.5) & 90.3 (87.5; 92.5) & {91.4 (89.2; 93.9)} \\
\bottomrule
\end{tabular}
\label{tab:marginal_coverage_locus}
\end{table}

\begin{table}[h]
\caption{\textbf{Acceptance rate for (\locus{}).}
Acceptance probability, $\Prbb(X \in A_\lambda)$ (\%), over the test split.
\locus{} uses the score $U_\alpha$ with fixed calibration level $\alpha=10\%$.
Entries are the median over 30 runs, with the 5th and 95th percentiles in parentheses.}
\begin{tabular}{llllll}
\toprule
\textbf{Dataset} & \textbf{BART} (\%) & \textbf{BART ($\gamma$)} (\%) & \textbf{MC Dropout} (\%) & \textbf{MC Dropout($\gamma$)} (\%) \\
\midrule
airfoil & 6.3 (0.0; 26.2) & \textbf{9.9 (0.7; 23.8)} & 2.6 (0.0; 13.1) & 1.0 (0.0; 23.2) \\
bike & \textbf{28.0 (25.3; 29.0)} & 26.5 (25.2; 28.2) & 15.5 (12.5; 19.2) & 13.8 (11.8; 16.8) \\
concrete & \textbf{0.5 (0.0; 5.8)} & 0.0 (0.0; 1.9) & 0.0 (0.0; 0.0) & 0.0 (0.0; 1.0) \\
cycle & \textbf{0.0 (0.0; 0.5)} & \textbf{0.0 (0.0; 0.1)} & \textbf{0.0 (0.0; 0.4)} & \textbf{0.0 (0.0; 0.0)} \\
electric & 1.8 (0.3; 3.2) & \textbf{2.8 (1.0; 4.4)} & 0.7 (0.1; 2.5) & 0.3 (0.0; 2.7) \\
homes & 15.8 (12.1; 18.6) & \textbf{16.1 (12.7; 18.2)} & 14.0 (9.8; 18.0) & 11.7 (9.3; 14.6) \\
meps19 & 4.7 (0.3; 8.4) & 3.9 (0.3; 19.0) & \textbf{33.1 (25.0; 39.1)} & 31.7 (27.2; 37.5) \\
protein & 0.0 (0.0; 0.0) & 0.0 (0.0; 0.0) & \textbf{11.2 (9.4; 12.7)} & 9.1 (7.4; 11.5) \\
star & \textbf{0.0 (0.0; 1.4)} & \textbf{0.0 (0.0; 0.5)} & \textbf{0.0 (0.0; 3.1)} & \textbf{0.0 (0.0; 0.5)} \\
superconductivity & 28.3 (25.0; 32.7) & 26.5 (22.6; 30.3) & \textbf{34.6 (31.3; 37.4)} & 29.2 (26.0; 33.6) \\
wec & 5.3 (3.0; 9.8) & 6.7 (2.9; 9.3) & \textbf{52.9 (50.6; 54.6)} & 49.8 (47.1; 51.6) \\
winered & \textbf{0.0 (0.0; 5.7)} & \textbf{0.0 (0.0; 4.2)} & \textbf{0.0 (0.0; 3.9)} & \textbf{0.0 (0.0; 2.6)} \\
winewhite & \textbf{0.0 (0.0; 0.3)} & \textbf{0.0 (0.0; 1.1)} & \textbf{0.0 (0.0; 0.4)} & \textbf{0.0 (0.0; 0.2)} \\

\bottomrule
\end{tabular}
\label{tab:acceptance_locus_tau}
\end{table}

\begin{table}[h]
\caption{\textbf{Conditional large-loss rate.}
Conditional exceedance among accepted points, $\Prbb(Z>\tau \mid X\in A_\lambda)$ (\%) over the test split.
\locus{} uses the score $U_\alpha$ with fixed calibration level $\alpha=10\%$.
Entries are the median over 30 runs, with the 5th and 95th percentiles in parentheses; lower is better.
Bold marks the lowest value in each row.}\begin{tabular}{llllll}
\toprule
\textbf{Dataset} & \textbf{BART} (\%) & \textbf{BART ($\gamma$)} (\%) & \textbf{MC Dropout} (\%) & \textbf{MC Dropout($\gamma$)} (\%) \\
\midrule
airfoil & \textbf{0.0 (0.0; 11.5)} & \textbf{0.0 (0.0; 18.7)} & 7.1 (0.0; 30.7) & 6.9 (0.0; 15.2) \\
bike & 1.6 (0.7; 2.4) & \textbf{1.0 (0.4; 1.7)} & 2.2 (0.7; 4.4) & 1.9 (0.7; 4.1) \\
concrete & \textbf{0.0 (0.0; 65.0)} & \textbf{0.0 (0.0; 50.0)} & 100.0 (100.0; 100.0) & \textbf{0.0 (0.0; 90.0)} \\
cycle & 33.3 (2.5; 86.7) & 100.0 (10.0; 100.0) & \textbf{0.0 (0.0; 0.0)} & -- (--; --) \\
electric & 4.2 (0.0; 20.5) & 4.8 (0.0; 15.9) & 16.7 (0.0; 37.3) & \textbf{0.0 (0.0; 31.2)} \\
homes & 5.0 (3.5; 7.0) & 5.6 (3.3; 6.6) & 4.9 (2.9; 6.9) & \textbf{4.0 (2.0; 5.5)} \\
meps19 & 2.6 (0.0; 6.7) & \textbf{2.4 (0.0; 7.7)} & 5.3 (4.0; 6.8) & 4.5 (2.9; 6.0) \\
protein & \textbf{0.0 (0.0; 0.0)} & -- (--; --) & 3.9 (2.3; 5.2) & 2.7 (1.2; 4.2) \\
star & \textbf{0.0 (0.0; 80.0)} & \textbf{0.0 (0.0; 0.0)} & \textbf{0.0 (0.0; 60.0)} & \textbf{0.0 (0.0; 100.0)} \\
superconductivity & 2.5 (1.4; 3.2) & \textbf{2.3 (1.2; 3.2)} & 3.0 (2.0; 4.1) & \textbf{2.3 (1.0; 3.0)} \\
wec & 4.3 (1.7; 8.1) & 2.9 (1.6; 7.0) & 2.3 (2.0; 3.1) & \textbf{2.1 (1.7; 2.7)} \\
winered & 10.0 (0.0; 80.0) & \textbf{0.0 (0.0; 77.8)} & \textbf{0.0 (0.0; 69.4)} & \textbf{0.0 (0.0; 0.0)} \\
winewhite & \textbf{0.0 (0.0; 28.3)} & \textbf{0.0 (0.0; 14.6)} & \textbf{0.0 (0.0; 68.1)} & \textbf{0.0 (0.0; 90.0)} \\
\bottomrule
\end{tabular}
\label{tab:condtail_locus_tau}
\end{table}

\begin{table*}[t]
\centering
\caption{\textbf{Tuned acceptance rate at matched target coverage (\locusLambda{}).}
Acceptance probability, $\Prbb(X \in A_\lambda)$ (\%).
For each dataset and method, the threshold $\lambda$ is selected on the validation split
to target an acceptance rate of $\approx 70\%$; test performance is then reported.
\locusLambda{} methods use the \locus{} score $U_\alpha$ with fixed calibration level $\alpha=10\%$.
Entries are the median over 30 runs, with the 5th and 95th percentiles in parentheses.}
\setlength{\tabcolsep}{4pt}
\renewcommand{\arraystretch}{1.05}
\scriptsize
\begin{tabular}{lcc cccc}
\toprule
& \multicolumn{2}{c}{\textbf{Baselines (no \locus{})}} 
& \multicolumn{4}{c}{\textbf{\locusLambda{} (loss-quantile score)}} \\
\cmidrule(lr){2-3}\cmidrule(lr){4-7}
\textbf{Dataset}
& \textbf{IFlag}
& \textbf{VARNet}
& \textbf{\locus{}-BART}
& \textbf{\locus{}-BART ($\gamma$)}
& \textbf{\locus{}-MC}
& \textbf{\locus{}-MC ($\gamma$)} \\
\midrule
airfoil & 73.2 (63.8; 80.8) & \textbf{82.8 (69.1; 95.9)} & 72.8 (65.1; 78.8) & 71.2 (61.5; 78.9) & 70.5 (61.2; 77.2) & 71.5 (62.5; 77.2) \\
bike & 73.0 (73.0; 73.0) & \textbf{88.2 (74.8; 94.3)} & 72.2 (70.9; 74.6) & 71.9 (70.5; 74.3) & 73.3 (71.9; 75.1) & 72.5 (71.0; 73.7) \\
concrete & 72.3 (66.0; 80.1) & \textbf{80.6 (73.8; 89.9)} & 72.3 (61.6; 83.2) & 72.3 (61.0; 85.9) & 72.8 (64.5; 86.1) & 74.3 (64.1; 80.7) \\
cycle & 72.4 (67.9; 76.2) & \textbf{92.2 (74.8; 98.9)} & 73.8 (69.2; 78.6) & 73.1 (68.4; 78.9) & 75.0 (70.8; 81.9) & 74.0 (68.8; 80.5) \\
electric & 71.5 (67.5; 75.9) & \textbf{79.0 (71.5; 83.9)} & 71.7 (68.6; 74.9) & 71.5 (69.1; 75.4) & 71.0 (68.2; 74.7) & 71.4 (69.1; 75.1) \\
homes & 71.9 (69.6; 74.3) & \textbf{80.8 (72.4; 88.1)} & 71.3 (69.3; 73.1) & 71.3 (69.4; 73.5) & 71.3 (69.2; 73.7) & 70.8 (68.5; 73.1) \\
meps19 & 71.7 (69.5; 74.5) & \textbf{79.4 (68.5; 84.8)} & 73.1 (69.3; 76.1) & 72.3 (69.7; 75.5) & 71.3 (68.8; 73.4) & 71.3 (68.8; 74.0) \\
protein & 72.3 (70.0; 74.5) & \textbf{73.1 (71.3; 79.2)} & 71.1 (69.0; 73.7) & 71.1 (69.6; 74.0) & 71.2 (68.7; 72.7) & 71.2 (69.9; 73.2) \\
star & 72.8 (63.9; 79.3) & 73.0 (65.2; 80.5) & \textbf{76.3 (67.0; 84.4)} & 75.1 (66.4; 83.5) & 73.5 (66.4; 77.3) & 74.7 (65.7; 81.4) \\
superconductivity & 71.5 (69.7; 74.0) & \textbf{82.4 (71.6; 88.8)} & 70.9 (69.3; 73.6) & 71.2 (69.2; 72.5) & 71.4 (69.2; 75.0) & 71.4 (68.6; 73.9) \\
wec & 71.6 (69.7; 74.0) & \textbf{88.1 (85.2; 89.7)} & 71.2 (69.6; 72.8) & 71.1 (69.4; 72.2) & 70.6 (69.0; 72.5) & 70.3 (69.1; 72.1) \\
winered & 70.9 (62.1; 77.6) & \textbf{73.4 (65.0; 80.6)} & 70.9 (63.1; 79.4) & 67.5 (59.0; 78.9) & 73.1 (62.6; 76.9) & 71.6 (61.7; 81.2) \\
winewhite & 72.9 (68.1; 76.4) & 72.7 (66.2; 79.4) & \textbf{80.2 (69.9; 85.2)} & 77.1 (68.6; 81.9) & 76.5 (69.7; 83.8) & 76.3 (69.8; 82.6) \\

\bottomrule
\end{tabular}
\label{tab:acceptance_locus_lambda}
\end{table*}

\begin{table*}[t]
\centering
\caption{\textbf{Tuned acceptance rate at matched target coverage (\locusAlpha{}).}
Acceptance probability, $\Prbb(X \in A_\lambda)$ (\%).
For each dataset and method, the threshold $\lambda$ is selected on the validation split
to target an acceptance rate of $\approx 70\%$; test performance is then reported.
\locusAlpha{} methods use the \locus{} score $U_\alpha$ with fixed calibration level $\alpha=10\%$.
Entries are the median over 30 runs, with the 5th and 95th percentiles in parentheses.}
\setlength{\tabcolsep}{4pt}
\renewcommand{\arraystretch}{1.05}
\scriptsize
\begin{tabular}{lcc cccc}
\toprule
& \multicolumn{2}{c}{\textbf{Baselines (no \locus{})}} 
& \multicolumn{4}{c}{\textbf{\locusAlpha{} (loss-quantile score)}} \\
\cmidrule(lr){2-3}\cmidrule(lr){4-7}
\textbf{Dataset}
& \textbf{IFlag}
& \textbf{VARNet}
& \textbf{\locus{}-BART}
& \textbf{\locus{}-BART ($\gamma$)}
& \textbf{\locus{}-MC}
& \textbf{\locus{}-MC ($\gamma$)} \\
\midrule
airfoil & 73.2 (63.8; 80.8) & \textbf{82.8 (69.1; 95.9)} & 71.5 (64.5; 78.1) & 71.5 (62.5; 79.2) & 70.5 (60.3; 77.5) & 71.5 (63.5; 82.3) \\
bike & 73.0 (73.0; 73.0) & \textbf{88.2 (74.8; 94.3)} & 71.4 (69.5; 73.3) & 71.7 (70.3; 73.6) & 72.3 (70.9; 74.0) & 70.9 (69.2; 72.8) \\
concrete & 72.3 (66.0; 80.1) & \textbf{80.6 (73.8; 89.9)} & 71.8 (62.4; 83.7) & 71.8 (62.5; 84.2) & 73.8 (62.0; 83.6) & 72.8 (64.5; 79.2) \\
cycle & 72.4 (67.9; 76.2) & \textbf{92.2 (74.8; 98.9)} & 71.6 (68.5; 75.1) & 71.6 (67.4; 75.8) & 72.6 (68.2; 79.1) & 75.8 (68.8; 85.3) \\
electric & 71.5 (67.5; 75.9) & \textbf{79.0 (71.5; 83.9)} & 71.6 (68.2; 74.2) & 70.8 (67.8; 73.4) & 71.0 (67.7; 74.0) & 71.1 (68.5; 75.2) \\
homes & 71.9 (69.6; 74.3) & \textbf{80.8 (72.4; 88.1)} & 71.6 (68.6; 72.9) & 70.5 (68.8; 73.5) & 70.9 (68.5; 73.5) & 70.7 (68.9; 72.8) \\
meps19 & 71.7 (69.5; 74.5) & \textbf{79.4 (68.5; 84.8)} & 71.1 (67.2; 72.3) & 70.4 (68.5; 72.3) & 70.0 (68.0; 72.0) & 70.2 (67.5; 71.9) \\
protein & 72.3 (70.0; 74.5) & \textbf{73.1 (71.3; 79.2)} & \textbf{73.1 (69.9; 76.5)} & 72.1 (70.5; 74.4) & 71.2 (69.5; 72.7) & 71.6 (69.2; 74.0) \\
star & 72.8 (63.9; 79.3) & 73.0 (65.2; 80.5) & \textbf{75.1 (66.6; 83.3)} & 73.5 (64.2; 80.6) & 73.5 (66.7; 81.2) & 72.6 (67.3; 79.6) \\
superconductivity & 71.5 (69.7; 74.0) & \textbf{82.4 (71.6; 88.8)} & 70.9 (68.9; 72.8) & 70.8 (68.6; 72.8) & 71.9 (69.2; 73.3) & 71.1 (69.3; 73.3) \\
wec & 71.6 (69.7; 74.0) & \textbf{88.1 (85.2; 89.7)} & 70.6 (69.2; 72.0) & 70.1 (68.9; 72.0) & 70.1 (68.5; 71.5) & 70.3 (69.0; 71.3) \\
winered & 70.9 (62.1; 77.6) & 73.4 (65.0; 80.6) & \textbf{74.4 (62.8; 80.1)} & 71.9 (61.3; 80.7) & 74.1 (60.8; 80.3) & 71.9 (62.8; 81.7) \\
winewhite & 72.9 (68.1; 76.4) & 72.7 (66.2; 79.4) & \textbf{75.1 (67.1; 78.0)} & 73.7 (67.4; 77.9) & 72.3 (65.5; 79.4) & 73.0 (67.6; 77.2) \\
\bottomrule
\end{tabular}
\label{tab:acceptance_locus_alpha}
\end{table*}

\begin{table*}[hp]
\centering
\caption{\textbf{Tuned conditional large-loss rate at matched acceptance (\locusAlpha{}).}
Conditional exceedance among accepted points, $\Prbb(Z>\tau \mid X\in A_\lambda)$ (\%).
For each dataset and method, the threshold $\lambda$ is selected on the validation split
to achieve an acceptance rate of $\approx 70\%$; test performance is then reported.
\locusAlpha{} methods use the \locus{} score $U_\alpha$ with fixed calibration level $\alpha=10\%$.
Entries are the median over 30 runs, with the 5th and 95th percentiles in parentheses; lower is better.
Bold marks the lowest value in each row.}
\label{tab:condtail_locus_alpha}
\scriptsize
\begin{tabular}{lcc cccc}
\toprule
& \multicolumn{2}{c}{\textbf{Baselines (no \locus{})}} 
& \multicolumn{4}{c}{\textbf{\locusAlpha{} (loss-quantile score)}} \\
\cmidrule(lr){2-3}\cmidrule(lr){4-7}
\textbf{Dataset}
& \textbf{IFlag}
& \textbf{VARNet}
& \textbf{\locus{}-BART}
& \textbf{\locus{}-BART ($\gamma$)}
& \textbf{\locus{}-MC}
& \textbf{\locus{}-MC ($\gamma$)} \\
\midrule
airfoil & 31.0 (27.4; 35.3) & 27.5 (23.0; 30.7) & 23.3 (19.6; 26.7) & \textbf{22.2 (18.4; 28.3)} & 24.5 (19.2; 29.3) & 24.7 (19.1; 27.4) \\
bike & 30.1 (30.1; 30.1) & 28.4 (26.1; 29.2) & \textbf{19.0 (18.3; 20.8)} & 19.4 (18.3; 20.6) & 22.1 (21.3; 22.7) & 21.8 (20.6; 23.1) \\
concrete & 33.5 (27.2; 38.6) & 27.7 (23.0; 31.2) & \textbf{23.7 (19.0; 28.0)} & 25.2 (20.1; 29.1) & 27.6 (20.0; 31.9) & 27.4 (21.8; 32.9) \\
cycle & 28.9 (27.2; 30.0) & 29.7 (28.5; 30.2) & \textbf{26.7 (25.2; 27.8)} & 27.1 (25.5; 28.2) & 28.4 (26.7; 29.6) & 28.9 (27.5; 30.1) \\
electric & 31.5 (30.1; 32.1) & 29.4 (28.1; 30.7) & \textbf{24.1 (22.2; 25.4)} & \textbf{24.1 (22.4; 25.8)} & 24.7 (22.9; 26.0) & 24.6 (23.2; 26.4) \\
homes & 34.8 (33.7; 36.0) & 26.6 (24.8; 27.7) & 19.2 (17.9; 20.1) & 19.3 (17.7; 20.6) & 19.0 (17.6; 20.2) & \textbf{18.9 (18.1; 20.0)} \\
meps19 & 36.4 (35.4; 37.4) & 24.8 (21.5; 26.6) & 14.0 (12.3; 15.0) & 13.8 (11.9; 14.5) & \textbf{13.0 (11.7; 14.3)} & \textbf{13.0 (11.9; 14.4)} \\
protein & 28.0 (27.4; 29.0) & 27.4 (25.8; 29.4) & 24.6 (23.5; 25.2) & 23.8 (22.9; 25.2) & 22.8 (21.4; 23.5) & \textbf{22.6 (21.9; 23.5)} \\
star & 29.1 (26.3; 31.9) & 29.0 (26.9; 33.8) & 29.1 (26.1; 31.9) & 29.3 (26.1; 32.4) & \textbf{28.5 (25.6; 31.5)} & 28.9 (26.5; 33.1) \\
superconductivity & 24.0 (23.1; 25.3) & 25.8 (21.3; 28.1) & \textbf{17.9 (17.0; 19.0)} & 18.1 (16.4; 19.2) & 18.3 (16.7; 19.7) & \textbf{17.9 (16.7; 19.3)} \\
wec & 34.8 (33.5; 35.3) & 23.0 (21.4; 24.6) & 15.1 (12.2; 18.4) & 14.6 (12.9; 18.7) & \textbf{8.2 (7.0; 9.2)} & \textbf{8.2 (7.4; 8.9)} \\
winered & 30.2 (26.6; 32.9) & 28.8 (24.9; 32.9) & 27.5 (24.0; 30.7) & \textbf{26.9 (23.3; 29.3)} & 27.5 (23.4; 30.5) & 27.5 (23.1; 31.5) \\
winewhite & 29.9 (28.2; 31.5) & 30.3 (28.3; 31.7) & 27.1 (24.6; 29.3) & \textbf{26.5 (24.3; 29.3)} & 27.8 (24.1; 29.4) & 27.4 (25.0; 29.2) \\
\bottomrule
\end{tabular}
\end{table*}

\end{document}